\documentclass[final]{cvpr}

\usepackage{times}
\usepackage{epsfig}
\usepackage{graphicx}
\usepackage{amsmath}
\usepackage{amssymb}

\usepackage{amsthm}
\usepackage{appendix}
\usepackage{xcolor, colortbl}
\definecolor{Gray}{gray}{0.85}
\definecolor{Gray1}{gray}{0.9}
\usepackage{subfig}
\newtheorem{lemma}{Lemma}
\newtheorem*{prop1*}{Property of Convergence}
\newtheorem*{prop2*}{Property of Monotonicity}

\usepackage[font=small,skip=0pt]{caption}
\newtheorem{prop}{Property}

\usepackage[pagebackref=true,breaklinks=true,colorlinks,bookmarks=false]{hyperref}

\def\cvprPaperID{250} % *** Enter the CVPR Paper ID here
\def\confYear{CVPR 2021}

\begin{document}

%%%%%%%%% TITLE
\title{MagFace: A Universal Representation for \\ Face Recognition and Quality Assessment}

\author{Qiang Meng, Shichao Zhao, Zhida Huang, Feng Zhou\\
  Algorithm Research, Aibee Inc. \\
 {\tt \small \{qmeng,sczhao,zdhuang,fzhou\}@aibee.com}
}

\maketitle
\pagestyle{empty}
\thispagestyle{empty}

%%%%%%%%% ABSTRACT
\begin{abstract}

  The performance of face recognition system degrades when the variability of the acquired faces increases.
  Prior work alleviates this issue by either monitoring the face quality in pre-processing or predicting the data uncertainty along with the face feature.
  This paper proposes MagFace, a category of losses that learn a universal feature embedding whose magnitude can measure the quality of the given face.
  Under the new loss, it can be proven that the magnitude of the feature embedding monotonically increases if the subject is more likely to be recognized.
  In addition, MagFace introduces an adaptive mechanism to learn a well-structured within-class feature distributions by pulling easy samples to class centers while pushing hard samples away.
  This prevents models from overfitting on noisy low-quality samples and improves face recognition in the wild.
  Extensive experiments conducted on face recognition, quality assessments as well as clustering demonstrate its superiority over state-of-the-arts.
  The code is available at \url{https://github.com/IrvingMeng/MagFace}.

  % Face quality assessment aims to estimate the suitability of a face image for recognition. Most previous quality models, which heavily relied on artificially or human labelled quality values, ignore the gaps between human-defined quality and real recognition characteristics. In addition, lacking of a standard and clear definition of quality leads to a lot of noisy labels.
  % In this work, we propose a loss called MagFace  which learns unified features for face recognition, clustering and quality assessment.
  % For face recognition, our MagFace prevents model overfitting on noisy and low-quality samples by an adaptive mechanism.
  % For face clustering, MagFace pulls easy samples close to class center while pushes hard samples away, which benefits the task by the well-structured feature distributions.
  % For face quality assessment, qualities are directly bundled to the characteristics of recognition without any labels involved and revealed by the magnitudes of features.
  % MagFace, which modified from the state-of-the-art recognition model called ArcFace, is trained by normal recognition datasets/schedule  and  therefore easy to implement and use.
  % Besides that, we provide detailed mathematical proofs and experimental results to demonstrate that MagFace works as expected.

  % The experimental results show the superiority of the proposed method. \qtext{XXX}
\end{abstract}

%%%%%%%%% BODY TEXT
\section{Introduction}
\vspace{-1pt}

\begin{figure}[!htb]
  \centering
  \subfloat[\label{fig:intro0}]{\includegraphics[width=0.24\textwidth]{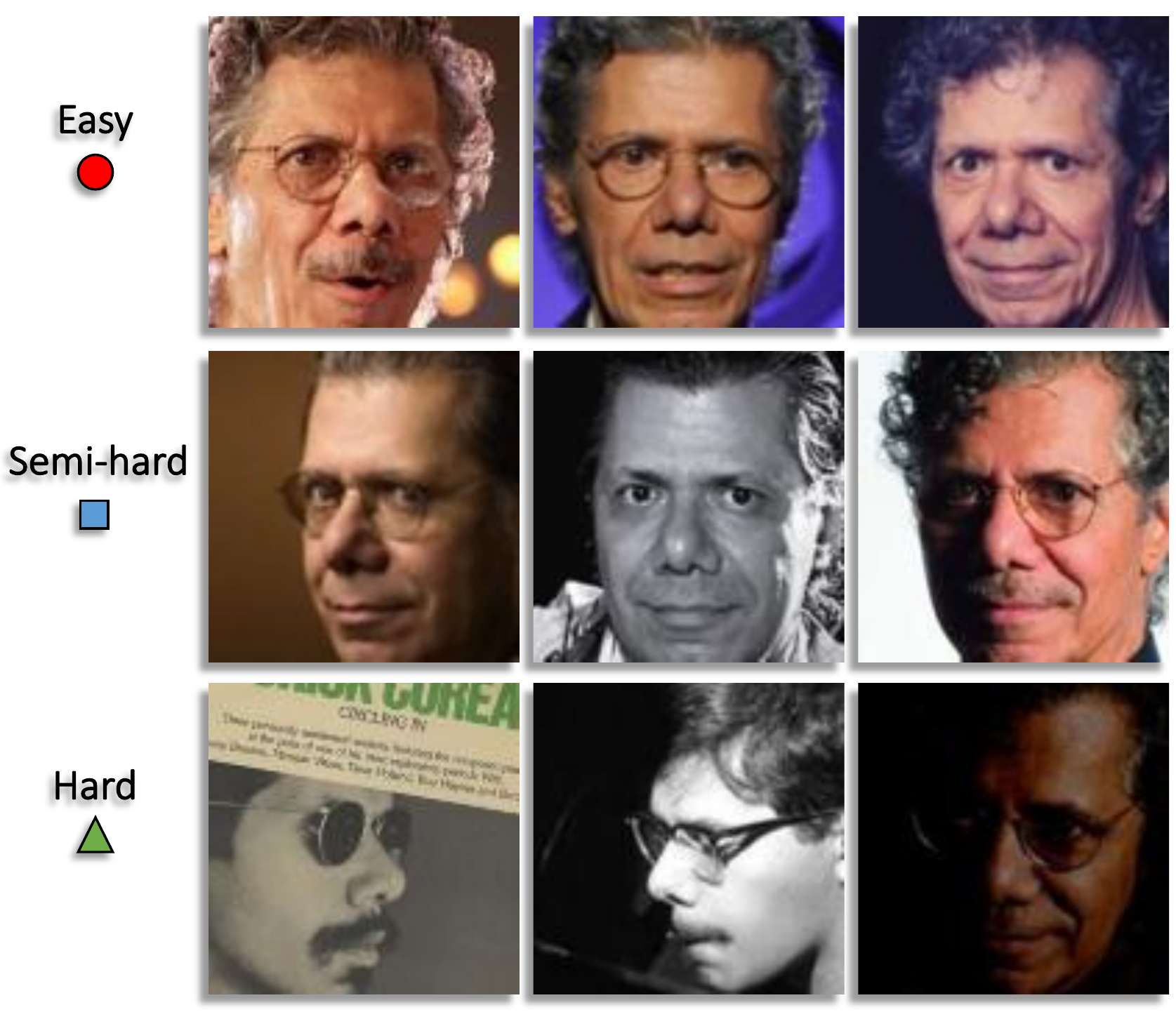}}
  \subfloat[\label{fig:intro1}]{\includegraphics[width=0.24\textwidth]{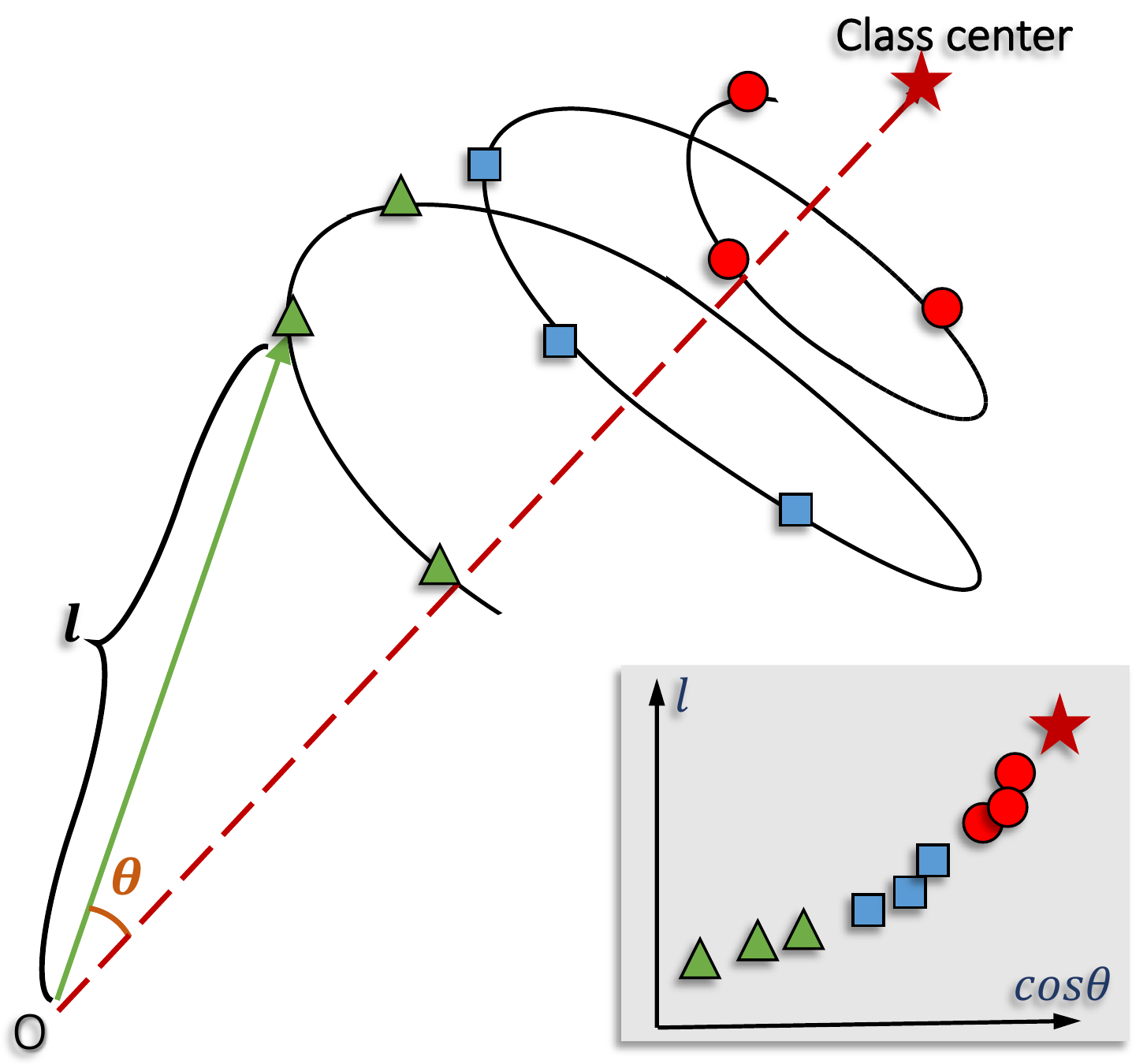}} % \\
  % \subfloat[Distributions from viewpoint A\label{fig:intro2}]{\includegraphics[width=0.15\textwidth, height=0.1\textwidth]{intro2.pdf}}  \hspace{5em}
  % \subfloat[Relationships between $\cos\theta$ and $l$\label{fig:intro3}]{\includegraphics[width=0.15\textwidth, height=0.1\textwidth]{intro3.pdf}}
  \caption{MagFace learns for (a) in-the-wild faces (b) a universal embedding by pulling the easier samples closer to the class center and pushing them away from the origin $o$% while pushing the harder ones away
    .
    % As shown in our experiments and supported by mathematical proof, the magnitude $l$ before normalization reveals the quality for each face as $l$ increases along with its cosine similarly with respect to the class center.
    As shown in our experiments and supported by mathematical proof, the magnitude $l$ before normalization increases along with feature's cosine distance to its class center, and therefore reveals the quality for each face.
    The larger the $l$, the more likely the sample can be recognized.}
  \label{fig:intro}
\end{figure}

Recognizing face in the wild is difficult mainly due to the large variability exhibited by face images acquired in unconstrained settings.
This variability is associated to the image acquisition conditions (such as illumination, background, blurriness, and low resolution), factors of the face (such as pose, occlusion and expression) or biases of the deployed face recognition system~\cite{Kolf1}.
To cope with these challenges, most relevant face analysis system under unconstrained environment (\eg, surveillance video) consists of three stages:
% locating a face with a rough bounding box and several fiducial landmarks,aligning the face image using a pre-defined template,
1) \textbf{face acquisition} to select from a set of raw images or capture from video stream the most suitable face image for recognition purpose;
2) \textbf{feature extraction} to extract discriminative representation from each face image;
3) \textbf{facial application} to match the reference image towards a given gallery or cluster faces into groups of same person.

To acquire the optimal reference image in the first stage, a technique called face quality assessment~\cite{Best-Rowden2017,schlett2020face} is often employed on each detected face.
Although the ideal quality score should be indicative of the face recognition performance, most of early work~\cite{ISO, ICAO} estimates qualities based on human-understandable factors such as luminances, distortions and pose angles, which may not directly favor the face feature learning in the second stage.
Alternatively, learning-based methods~\cite{Best-Rowden2017, Hernandez-ortega} train quality assessment models with artificially or human labelled quality values.
Theses methods are error-prone as there lacks of a clear definition of quality and human may not know the best characteristics for the whole systems.

% we have witnessed significant progresses on learning robust face representations due to the advances of convolutional neural networks (CNNs).
To achieve high end-to-end application performances in the second stage, various metric-learning~\cite{schroff2015facenet, sohn2016improved} or classification losses~\cite{Yuan2018,Ranjan2017,Liu2017,guo2020learning, Wang2018,Deng2018,cao2020domain} emerged in the past few years.
These works learn to represent each face image as a deterministic point embedding in the latent space regardless of the variance inherent in faces.
In reality, however, low-quality or large-pose images like Fig.~\ref{fig:intro0} widely exist and their facial features are ambiguous or absent.
Given these challenges, a large shift in the embedded points is inevitable, leading to false recognition.
For instance, performance reported by prior state-of-the-art~\cite{Shi2019} on IJB-C is much lower than LFW.
Recently, confidence-aware methods~\cite{Shi2019,Chang2020} propose to represent each face image as a Gaussian distribution in the latent space, where the mean of the distribution estimates the most likely feature values while the variance shows the uncertainty in the feature values.
Despite the performance improvement, these methods seek to separate the face feature learning from data noise modeling.
Therefore, additional network blocks are introduced in the architecture to compute the uncertainty level for each image.
This complicates the training procedure and adds computational burden in inference.
% In addition, the uncertainty measure cannot be directed used in the conventional $L_2$ metric for comparing face features.
In addition, the uncertainty measure cannot be directed used in conventional metrics for comparing face features.

% The first stage is important for achieving high performances as face quality is a key factor for robust face representation, especially when dealing with large variabilities which leads to significant performance degradations.

% To further improve the discriminative power, mining-based losses~\cite{shrivastava2016training, lin2017focal, Wang2018_sv} focus on hard examples, which are considered as informative examples in their sides.

This paper proposes MagFace to learn a universal and quality-aware face representation.
The design of MagFace follows two principles:
1) Given the face images of the same subject but in different levels of quality (\eg, Fig.~\ref{fig:intro0}), it seeks to learn a within-class distribution, where the high-quality ones stay close to the class center while the low-quality ones are distributed around the boundary.
2) It should pose the minimum cost for changing existing inference architecture to measure the face quality along with the computation of face feature.
To achieve the above goals, we choose magnitude, the independent property to the direction of the feature vector, as the indicator for quality assessment.
The core objective of MagFace is to not only enlarge inter-class distance, but also maintain a cone-like within-class structure like Fig.~\ref{fig:intro1}, where % high-quality samples are of large magnitude and
ambiguous samples are pushed away from the class centers and pulled to the origin.
This is realized by adaptively down-weighting ambiguous samples during training and rewarding the learned feature vector with large magnitude in the MagFace loss.
To sum up, MagFace improves previous work in two aspects:

\begin{enumerate}
  \itemsep0em

\item For the first time, MagFace explores the complete set of two properties associated with feature vector, direction and magnitude, in the problem of  face recognition while previous works often neglect the importance of the magnitude by normalizing the feature.
  With extensive experimental study and solid mathematical proof, we show that the magnitude can reveal the quality of faces and can be bundled with the characteristics of recognition without any quality labels involved.

\item MagFace explicitly distributes features structurally in the angular direction (as shown in Fig.~\ref{fig:intro1}).
By dynamically assigning angular margins based on samples' hardness for recognition, MagFace prevents model from overfitting on noisy and low-quality samples and learns a well-structured distributions that are more suitable for recognition and clustering purpose.

  % \item To the best of our knowledge, we are the first to systematically analyze the relationships between the feature magnitudes and image qualities for the softmax loss and cosine loss (\textit{e.g.}, CosFace \cite{Wang2018}, ArcFace \cite{Deng2018}). For our MagFace, we provide detailed mathematically proofs to show  feature magnitudes increase with smaller intra-class distances and larger inter-class distances.

  %   As quality assessment serves to estimate the suitability of an image for recognition, qualities should be coupled with the characteristic of recognition models. To this end, we design the MagFace loss which estimates qualities in an unsupervised manner.

% \item We conduct extensive experiments on face recognition, quality assessments and face clustering. The results verify the superiority of our MagFace compared to state-of-the-arts methods for all three tasks.

\end{enumerate}

\section{Related Works}
\vspace{-1pt}

\begin{figure*}[!htb]
  \centering
  \subfloat[]{\includegraphics[width=0.25\textwidth]{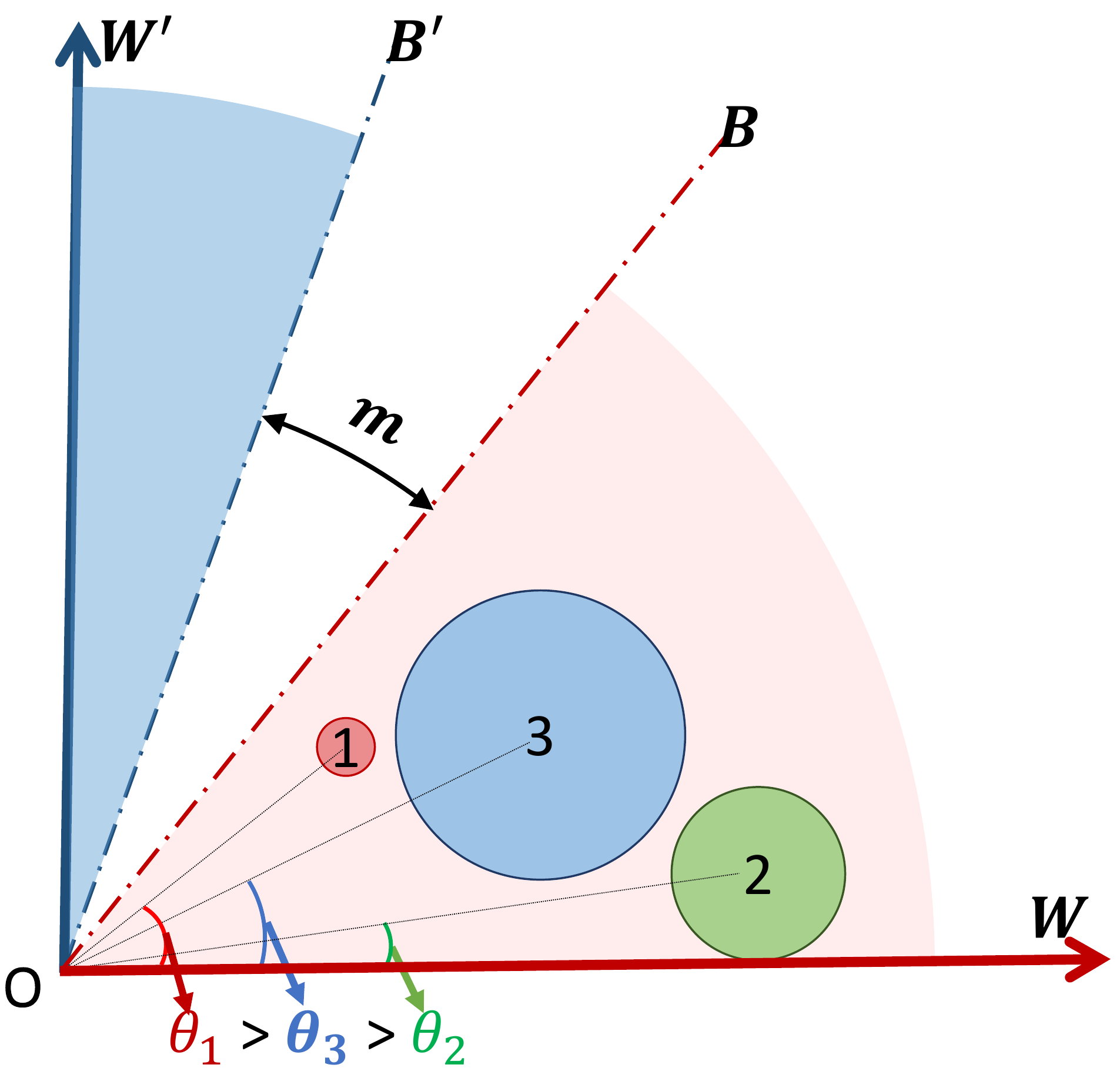}\label{fig:method0}}
  \subfloat[]{\includegraphics[width=0.25\textwidth]{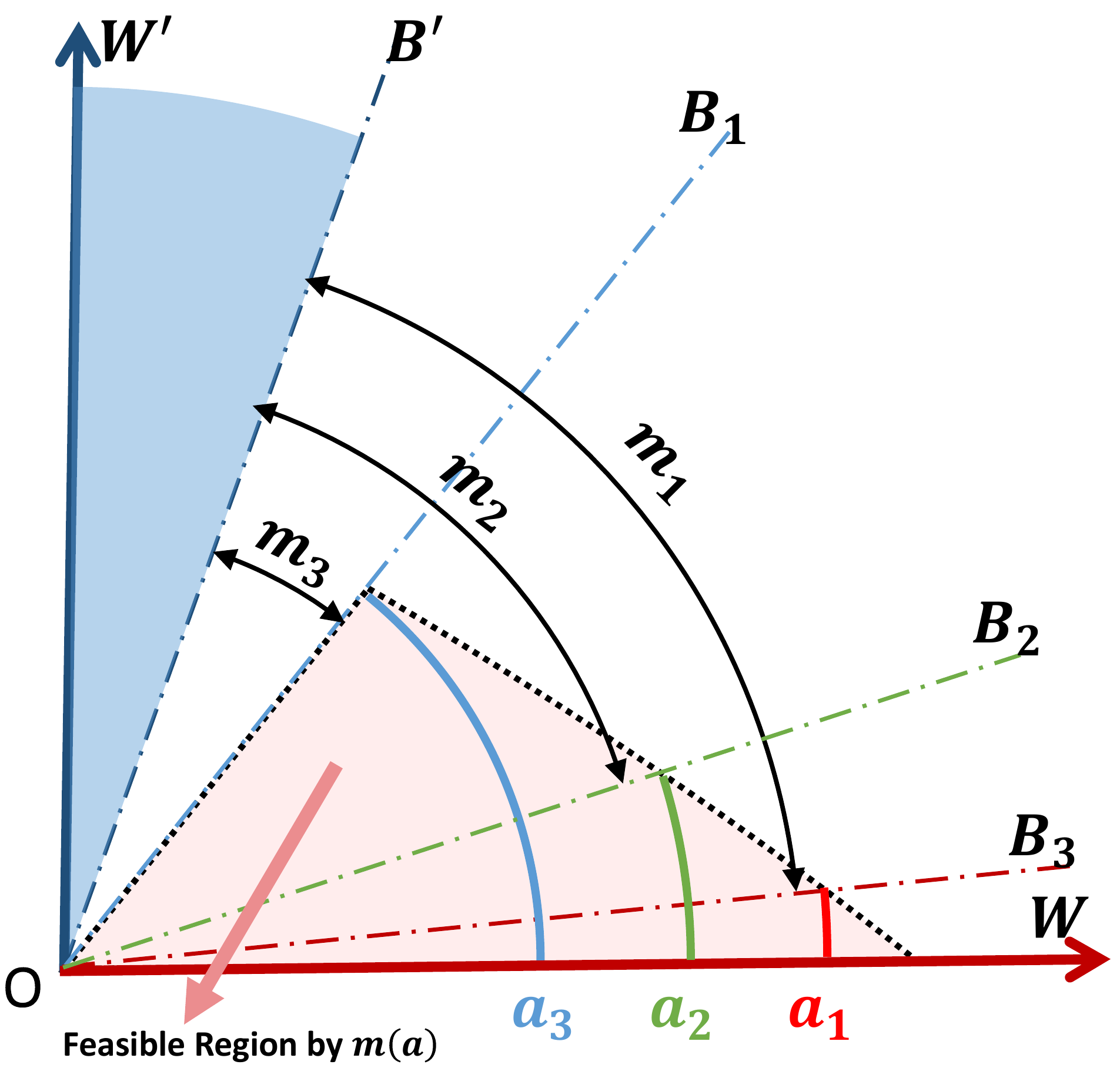}\label{fig:method1}}
  \subfloat[]{\includegraphics[width=0.25\textwidth]{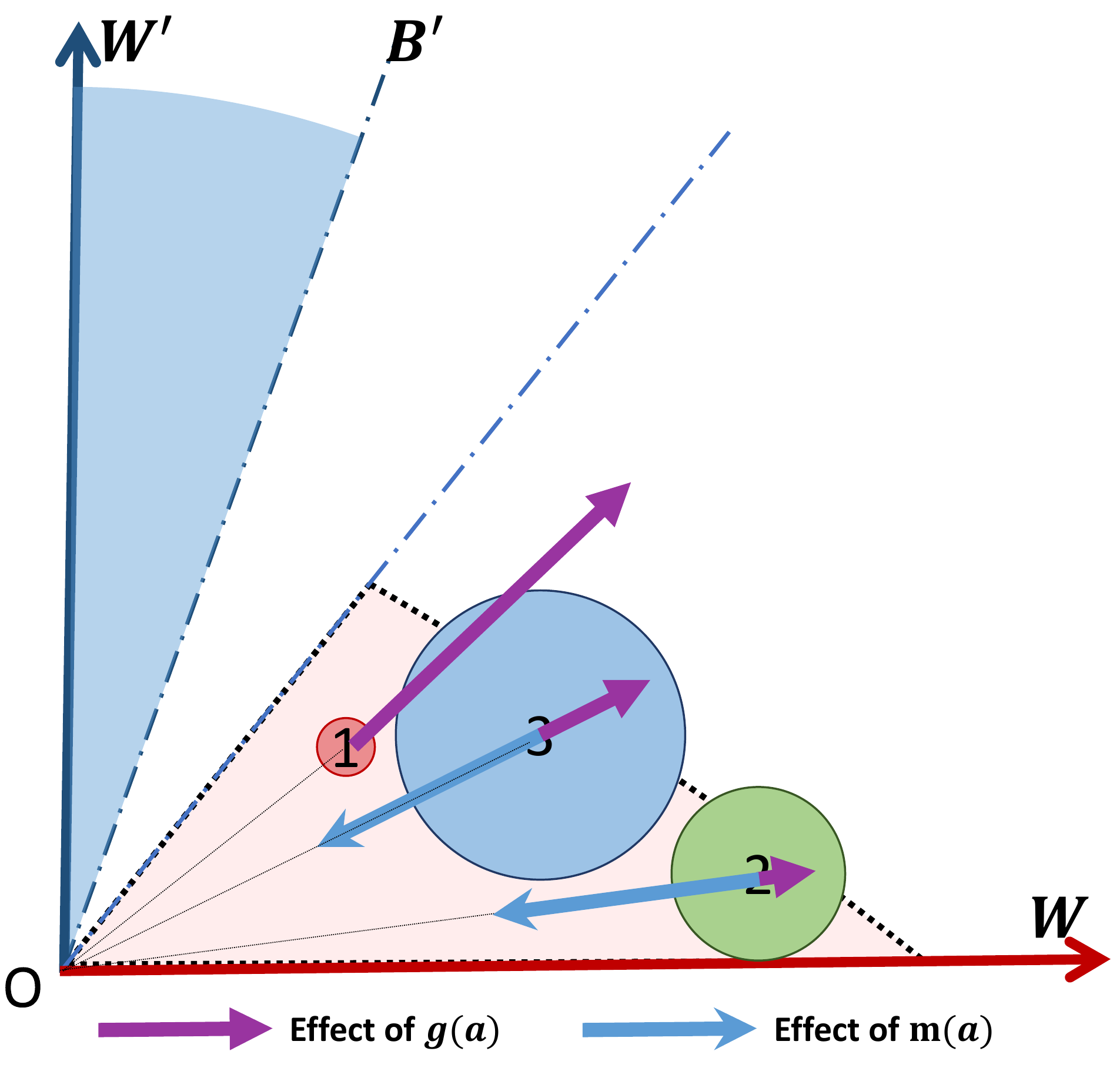}\label{fig:method2}}
  \subfloat[]{\includegraphics[width=0.25\textwidth]{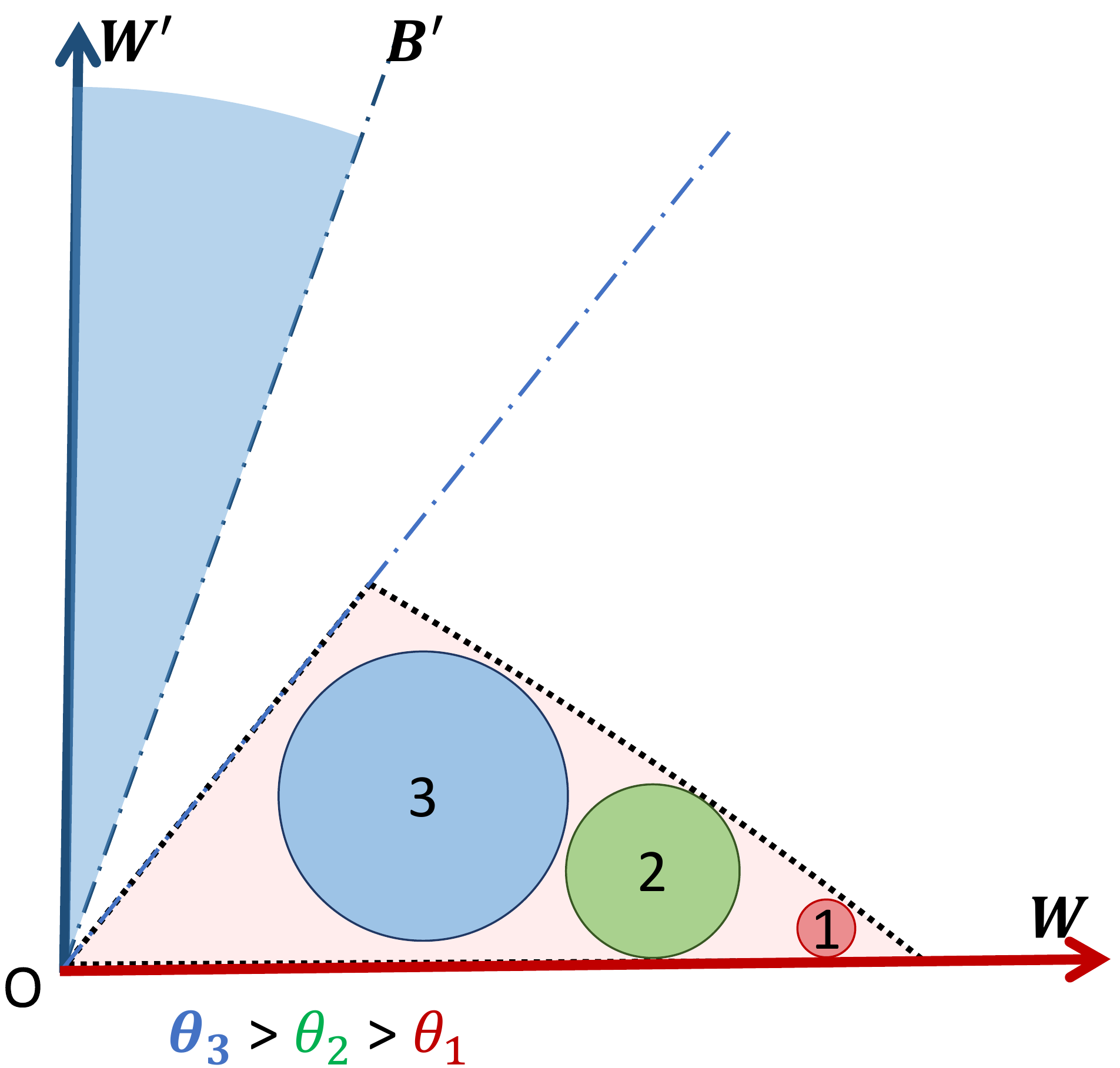}\label{fig:method3}}
  \caption{Geometrical interpretation of the feature space (without normalization) optimized by ArcFace and MagFace.
    (a) Two-class distributions optimized by ArcFace, where $w$ and $w'$ are the class centers and their decision boundaries $B$ and $B'$ are separated by the additive margin $m$.
    Circle 1, 2, 3 represent three types samples of class $w$ with descending qualities.
    % : smaller radius implies its uncertainty is lower and quality is higher.
    (b) MagFace introduces $m(a_i)$ which dynamically adjust boundaries based on feature magnitudes, and ends to a new feasible region.
    % We define  $m(a_i)$ as an increasing convex function of $a_i$.
    % If the feature magnitude of a sample is $a_1$, then it has  boundary $B_1$ and can be distributed in the blue line.
    % When the magnitude increases, the boundary gets close to class center $w_a$ and its feasible region shrinks (\eg, red line for a large magnitude $a_3$).
    % In the end, $m(a)$ results in a new feasible region as shown in the figure.
    (c) Effects of $g(a_i)$ and $m(a_i)$.
    % g(a) encourages large feature magnitudes as shown in purple arrows. However, features with too large magnitude will be out of the feasible region and be penalized by the classification loss.
    (d) Final feature distributions of our MagFace.
    % Hard samples will have large feature magnitudes and have large angular distances to class centers.
    \textbf{Best viewed in color.}}
  \label{fig:method}
\end{figure*}

\subsection{Face Recognition}
\vspace{-3pt}
Recent years have witnessed the breakthrough of deep convolutional face recognition techniques.
A number of successful systems, such as DeepFace~\cite{taigman2014deepface}, DeepID~\cite{sun2015deepid3}, FaceNet~\cite{schroff2015facenet} have shown impressive performance on face identification and verification.
Apart from the large-scale training data and deep network architectures, the major advance comes from the evolution of training losses for CNN.
Most of early works rely on metric-learning based loss, including contrastive loss~\cite{chopra2005learning}, triplet loss~\cite{schroff2015facenet}, n-pair loss~\cite{sohn2016improved}, angular loss~\cite{wang2017deep}, \etc.
Suffering from the combinatorial explosion in the number of face triplets, embedding-based method is usually inefficient in training on large-scale dataset.
Therefore, the main body of research in deep face recognition has focused on devising more efficient and effective classification-based loss.
Wen \etal~\cite{wen2016discriminative} develop a center loss to learn centers for each identity to enhance the intra-class compactness.
$L_2$-softmax~\cite{Ranjan2017} and NormFace~\cite{Wang2017} study the necessity of the normalization operation and applied $L_2$ normalization constraint on both features and weights.
From then on, several angular margin-based losses, such as SphereFace~\cite{Liu2017}, AM-softmax~\cite{wang2018additive}, SV-AM-Softmax~\cite{Wang2018_sv}, CosFace~\cite{Wang2018}, ArcFace~\cite{Deng2018}, progressively improve the performance on various benchmarks to the newer level.
More recently, AdaptiveFace~\cite{liu2019adaptiveface}, AdaCos~\cite{zhang2019adacos} and FairLoss~\cite{liu2019fair} introduce adaptive margin strategy to automatically tune hyperparameters and generate more effective supervisions during training.
Compared to our method, all these work tend to suppress the effect of magnitude in the loss by normalizing the feature vector.

% Yuan et al.~\cite{Yuan2018} proposed feature incay to regularize representation learning and push away features from the origin, which achieves better inter-class separability.
% For example,  modified the softmax loss to adaptively emphasize the mis-classified points (support vectors).

% Another type is angular and cosine margin based loss.
% SphereFace~\cite{Liu2017} introduced angular margin to softmax loss and achieved angularly discriminative features. To overcome the optimization difficulty of SphereFace, CosFace~\cite{Wang2018} moves the angular margin into cosine space.
% In , decision boundary is directly maximized in angular (arc) space based on the L2 normalized weights and features, and they achieve state-of-the-art performances on current benchmarks.

\subsection{Face Quality Assessment}
\vspace{-3pt}
Face image quality is an important factor to enable high-performance face recognition systems~\cite{Best-Rowden2017}.
Traditional methods, such as ISO/IEC 19794-5 standard~\cite{ISO}, ICAO 9303 standard~\cite{ICAO}, Brisque~\cite{sun2015no}, Niqe~\cite{Mittal2013} and  Piqe~\cite{venkatanath2015blind}, describe qualities from image-based aspects (\eg, distortion, illumination and occlusion) or subject-based measures (\eg, accessories).
Learning-based approaches such as FaceQNet~\cite{Hernandez-ortega} and Best-Rowden~\cite{Best-Rowden2017}  regress qualities by networks trained on human-assessed and similarity-based labels.
However, these quality labels are error-prone as human may not know the best characteristics for the recognition system and therefore cannot consider all proper factors.
Recently, several uncertainty-based methods are proposed to express face qualities by the uncertainties of features. SER-FIQ~\cite{Kolf1} forwards an image to a network with dropout several times and  measures face quality by the variation of extracted features.
% To be more specific, one image is passed to the recognition model many times with dropout layers active and variances of features are treated as qualities.
Confidence-aware face recognition methods~\cite{Shi2019,Chang2020} propose to represent each face image as a Gaussian distribution in the latent space and learn the uncertainty in the feature values.
Although these methods work in an unsupervised manner like ours, they require additional computational costs or network blocks, which complicate their usage in conventional face systems.

\subsection{Face Clustering}
\vspace{-3pt}
Face clustering exploits unlabeled data to cluster them into pseudo classes.
Traditional clustering methods usually work in an unsupervised manner, such as K-means~\cite{lloyd1982least}. DBSCAN~\cite{ester1996density} and hierarchical clustering.
% However, they lack the capability of coping with complicated cluster structures, thus often giving rise to noisy clusters, especially when applied to large-scale datasets collected from real-world settings.
% To address this issue, several supervised clustering methods are proposed recently.
Several supervised clustering methods based on graph convolutional network (GCN) are proposed recently.
For example, L-GCN~\cite{wang2019linkage} performs reasoning and infers the likelihood of linkage between pairs in the sub-graphs.
Yang \etal~\cite{yang2020learning} designs two graph convolutional networks, named GCN-V and GCN-E, to estimate the confidence of vertices and the connectivity of edges, respectively.
Instead of developing clustering methods, we aim at improving feature distribution structure for clustering.

\section{Methodology}
\vspace{-1pt}

\begin{figure*}[h]
  \centering
  \subfloat[Softmax]{\includegraphics[width=0.33\textwidth]{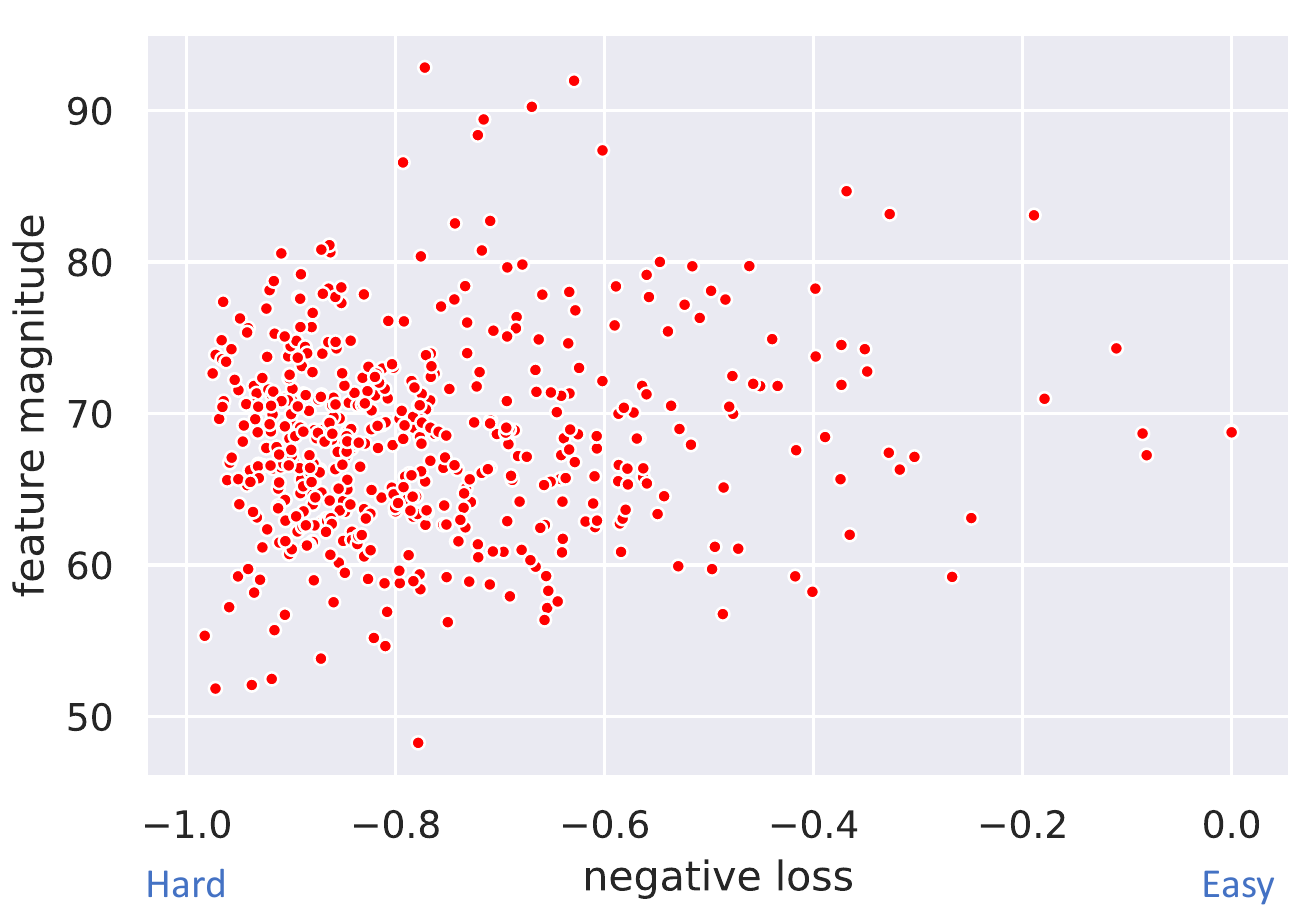}\label{fig:data1}}
  \subfloat[ArcFace]{\includegraphics[width=0.33\textwidth]{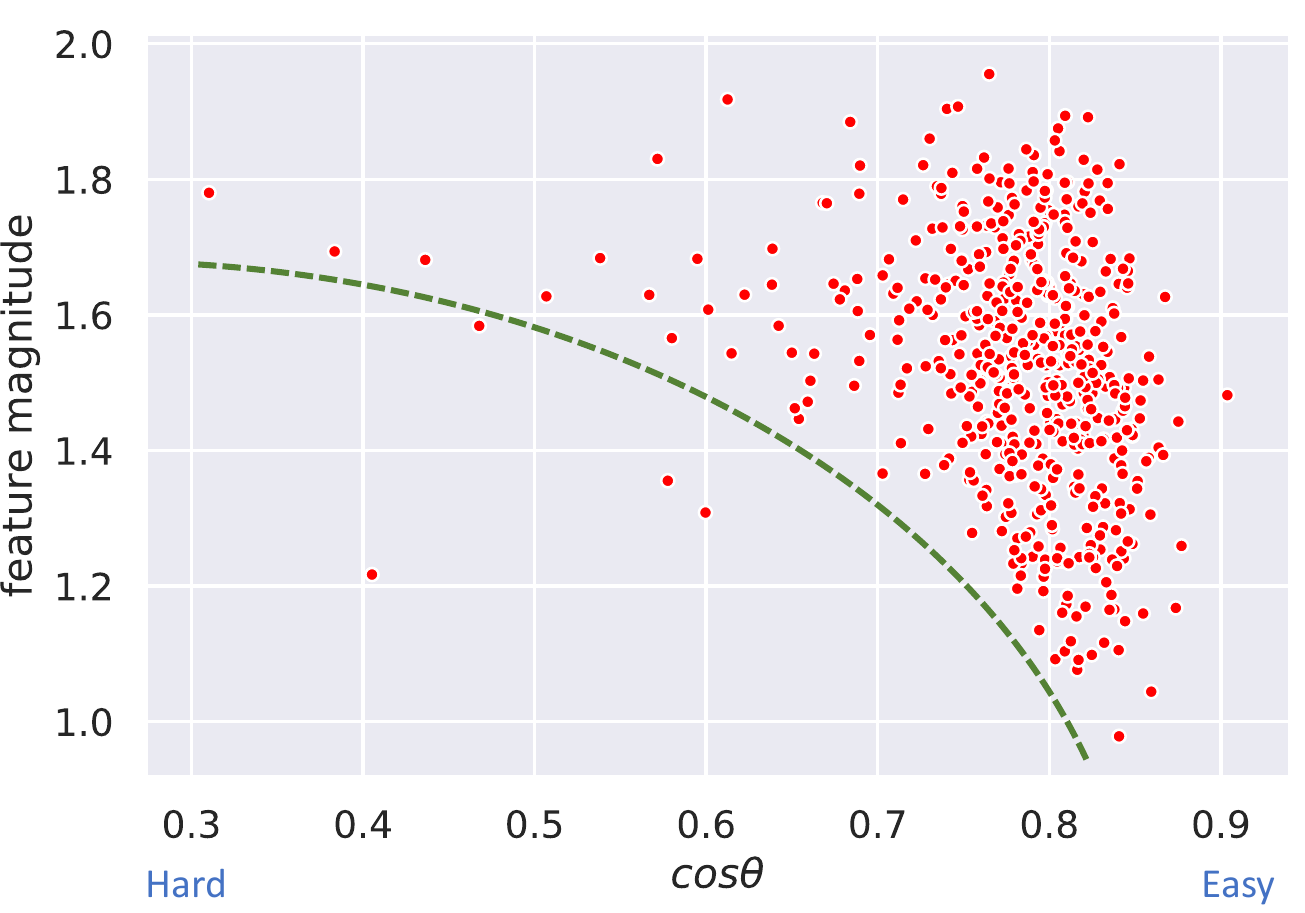}\label{fig:data2}}
  \subfloat[MagFace]{\includegraphics[width=0.33\textwidth]{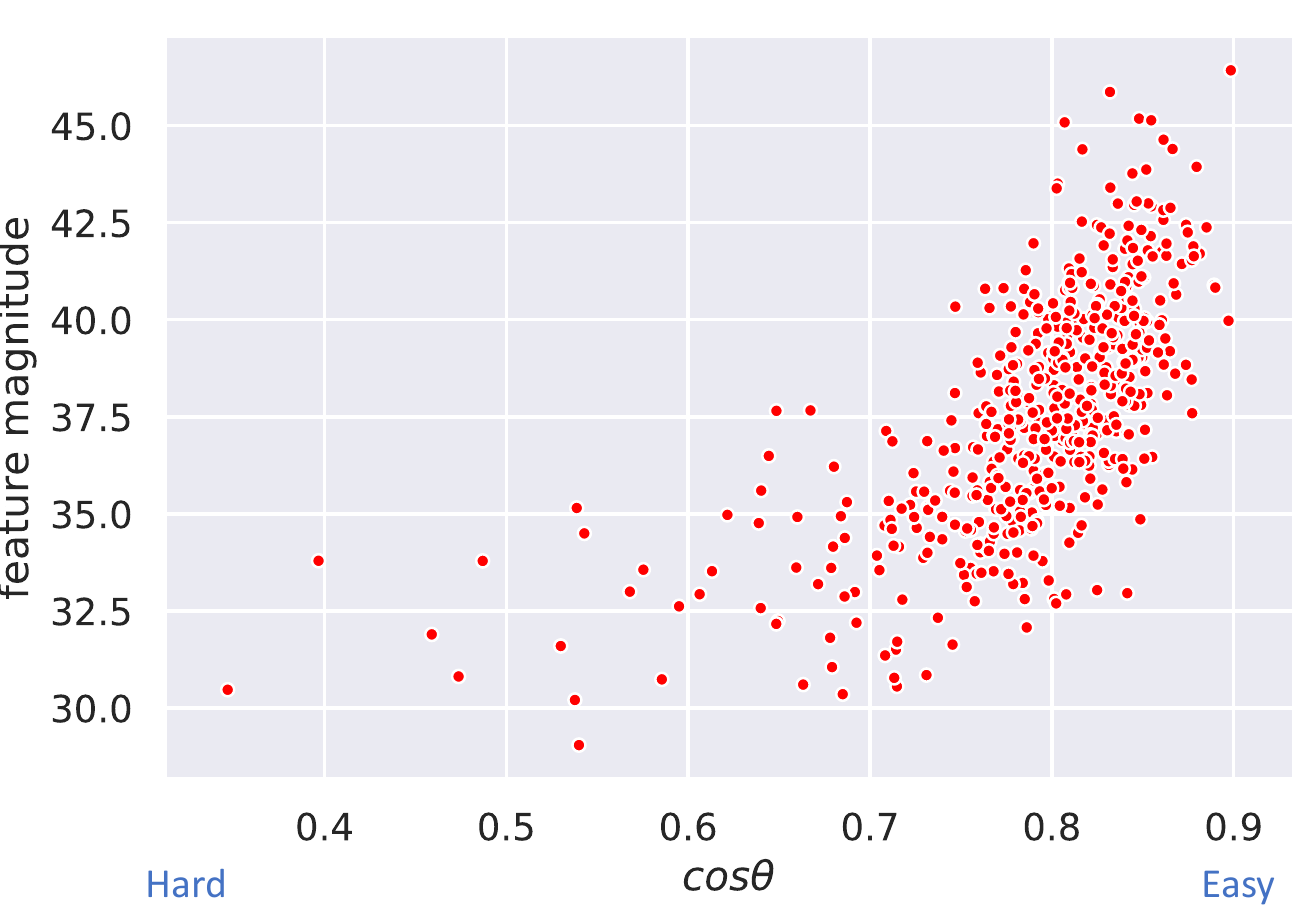}\label{fig:data3}}
  \caption{Visualization of feature magnitudes and  difficulties for recognition.
    % The networks (Backbone: ResNet50) are trained on MS1Mv2 \cite{guo2016ms,Deng2018} for  20 epochs with batch size 512 and initial learning rate 0.1, dropped by 0.1 every 5 epochs. 512 samples of the last iteration are used for visualization.
    Models are trained on MS1M-V2~\cite{guo2016ms,Deng2018} and 512 samples of the last iteration are used for visualization.
    Negative losses are used to reveal the hardness for Softmax while we use cosine value of $\theta$ (angle between a feature and its class center) for ArcFace and MagFace.}
  \label{fig:data_example}
\end{figure*}

In this section, we first review the definition of ArcFace~\cite{Deng2018}, one of the most popular losses used in face recognition.
Based on the analysis of ArcFace, we then derive the objective and prove the key properties for MagFace.
In the end, we compare softmax and ArcFace with MagFace from the perspective of feature magnitude.

\subsection{ArcFace Revisited}
\vspace{-3pt}

Training loss plays an important role in face representation learning.
Among the various choices (see \cite{du2020elements} for a recent survey), ArcFace~\cite{Deng2018} is perhaps the most widely adopted one in both academy and industry application due to its easiness in implementation and state-of-the-art performance on a number of benchmarks.
Suppose that we are given a training batch of $N$ face samples $\{f_i, y_i\}_{i=1}^N$ of $n$ identities, where $f_i \in \mathbb{R}^d$ denotes the $d$-dimensional embedding computed from the last fully connected layer of the neural networks and $y_i = 1, \cdots, n$ is its associated class label.
ArcFace and other variants improve the conventional softmax loss by optimizing the feature embedding on a hypersphere manifold where the learned face representation is more discriminative.
By defining the angle $\theta_j$ between $f_i$ and $j$-th class center $w_j \in \mathbb{R}^d$ as $w_j^T f_i = \|w_j\|\|f_i\|\cos\theta_j$, the objective of ArcFace~\cite{Deng2018} can be formulated as
\small
\begin{align}
    L = -\frac{1}{N}\sum_{i=1}^{N} \log \frac{e^{s\cos{(\theta_{y_i}+ m)}}}{ e^{s\cos{(\theta_{y_i}+m)}} + \sum_{j\neq y_i} e^{s \cos{\theta_j}}},
  \label{eq:arcloss}
\end{align}
\normalsize
where $m>0$ denotes the additive angular margin and $s$ is the scaling parameter.

Despite its superior performances on enforcing intra-class compactness and inter-class discrepancy, the angular margin penalty $m$ used by ArcFace is quality-agnostic and the resulting structure of the within-class distribution could be arbitrary in unconstrained scenarios.
For example, let us consider the scenario illustrated in Fig.~\ref{fig:method0}, where we have face images of the same class in three levels of qualities indicated by the circle sizes: the larger the radius, the more uncertain the feature representation and the more difficulty the face can be recognized.
Because ArcFace employs a uniform margin $m$, each image in one class shares the same decision boundary, \ie, $B: \cos(\theta + m) = \cos(\theta^{\prime})$ with respect to the neighbor class.
The three types of samples can stay at arbitrary location in the feasible region (shading area in Fig.~\ref{fig:method0}) without any penalization by the angular margin.
This leads to unstable within-class distribution, \eg, the high-quality face (type 1) stay along the boundary $B$ while the low-quality ones (type 2 and 3) are closer to the center $w$.
This unstableness can hurt the performances on in-the-wild recognition as well as other facial application such as face clustering.
Moreover, hard and noisy samples are over-weighted as they are hard to stay in the feasible area and the models may overfit to them.

\subsection{MagFace} \label{subsec:magface}
\vspace{-3pt}

Based on the above analysis, previous cosine-similarity-based face recognition loss lacks more fine-grained constraint beyond a fixed margin $m$.
This leads to unstable within-class structure especially in the unconstrained case (\eg, Fig.~\ref{fig:method0}) where the variability of each subject's faces is large.
To address the aforementioned problem, this section proposes MagFace, a novel framework to encode quality measure into the face representation.
Unlike previous work \cite{Shi2019,Chang2020} that call for additional uncertainty term, we pursue a minimalism design by optimizing over the magnitude $a_i = \|f_i\|$ without normalization of each feature $f_i$.
Our design has two major advantages:
1) We can keep using the cosine-based metric that has been widely adopted by most existing inference systems;
2) By simultaneously enforcing its direction and magnitude, the learned face representation is more robust to the variability of faces in the wild.
To our best understanding, this is the first work to unify the feature magnitude as quality indicator in face recognition.

Before defining the loss, let us first introduce two auxiliary functions related to $a_i$, the magnitude-aware angular margin $m(a_i)$ and the regularizer $g(a_i)$.
The design of $m(a_i)$ follows a natural intuition: for high-quality samples $x_i$, they should concentrate in a small region around the cluster center $w$ with high certainty.
By assuming a positive correlation between the magnitude and quality, we thereby penalize more on $x_i$ in terms of $m(a_i)$ if its magnitude $a_i$ is large.
% In theory, $m(a_i)$ can be any strictly increasing convex function.
% and we prove its convergence in the supplementary.
To have a better understanding, Fig.~\ref{fig:method1} visualizes the margins $m(a_i)$ corresponding to different magnitude values.
In contrast to ArcFace (Fig.~\ref{fig:method0}), the feasible region defined by $m(a_i)$ has a shrinking boundary with respect to feature magnitude towards the class center $w$.
Consequently, this boundary pulls the low-quality samples (circle 2 and 3 in Fig.~\ref{fig:method2}) to the origin where they have lower risk to be penalized.
However, the structure formed solely by $m(a_i)$ is unstable for high-quality samples like circle 1 in Fig.~\ref{fig:method2} as they have large freedom moving inside the feasible region.
We therefore introduce the regularizer $g(a_i)$ that rewards sample with large magnitude.
By designing $g(a_i)$ as a monotonically decreasing convex function with respect to $a_i$, each sample would be pushed towards the boundary of the feasible region and the high-quality ones (circle 1) would be dragged closer to the class center $w$ as shown in Fig.~\ref{fig:method3}.
In a nutshell, MagFace extends ArcFace (Eq.~\ref{eq:arcloss}) with magnitude-aware margin and regularizer to enforce higher diversity for inter-class samples and similarity for intra-class samples by optimizing:
\small
\begin{align}
    L_{Mag} &= \frac{1}{N}\sum_{i=1}^{N} L_i, \quad \text{where} \label{eq:magloss} \\
    L_i &= -\log \frac{e^{s\cos{(\theta_{y_i}+ m(a_i))}}}{ e^{s\cos{(\theta_{y_i}+m(a_i))}} + \sum_{j\neq y_i} e^{s \cos{\theta_j}}} + \lambda_g g(a_i). \nonumber
\end{align}
\normalsize
The hyper-parameter $\lambda_g$ is used to trade-off between the classification and regularization losses.

% and equip model with the ability to estimate qualities,
% For low-quality and noisy samples, MagFace down-weights their losses by an adaptive mechanism.
% In addition, their features are distributed away from class centers, as well as enforced to have small magnitudes.
% We have provided detailed mathematical proofs to ensure our MagFace to work as expected.
% To further understand feature magnitudes, we conduct experiments on  softmax, ArcFace and MagFace in Sec.~\ref{subsec:magnitude}.
% We provide detailed explanations for patterns observed and show that feature magnitudes of MagFace is a good metric for quality assessment.

% Denote the magnitude of feature $f_i$ is $a_i$, \textit{i.e.}, $a_i = \|f_i\|$ and set a pre-defined lower bound $l_a$ and upper bound $u_a$ for the feature magnitudes

The design of MagFace not only follows intuitive motivations, but also yields result with theoretical guarantees.
Assuming the magnitude $a_i$ is bounded in $[l_a, u_a]$, where $m(a_i)$ is a strictly increasing convex function, $g(a_i)$ is a strictly decreasing convex function and $\lambda_g$ is large enough, we can prove (see detailed requirements and proofs in the supplementary) that the following two properties of MagFace loss always hold when optimizing $L_i$ over $a_i$:

% As the magnitude and direction of a feature $f_i$ are two independent components, the optimal magnitude $a^*_i$ is achieved by solving the following problem with a fixed direction, (\textit{i.e.}, $\theta_{y_i}$ and $\theta_j, j\in\{1,\cdots, n\}, j\neq y_i$ are constant).

% If the sample $i$ is wrongly classified, the $L_i = - \log \frac{e^{s\cdot \cos(\theta_{y_i}+ m(a_i))}}{e^{s\cdot \cos(\theta_{y_i}+ m(a_i))} + \sum_{j=1, j\neq y_i}^{n}e^{s\cdot \cos{\theta_j}}}$ prefers a smallest $m(a_i)$, which leads to  $a_i=l_a$.

% In our work,  $m,g,\lambda_g$ are set to meet the following requirements:
% \begin{enumerate}
% \item $m(a_i)$ is an increasing convex function in $[l_a, u_a]$ and $m'(a_i) \in (0, k]$, where $k$ is a upper bound;
% \item $g(a_i)$ is a strictly convex function with $g'(u_a)=0$;
% \item $\lambda_g \geq \frac{sk}{-g'(l_a)}$.
% \end{enumerate}
% Then we can proves  that optimal magnitude $a^*_i$ has the following properties (detailed proofs in supplementary materials):
% \begin{enumerate}
% \item The probability of $\theta_{y_i}+ m(a_i)\in [0, \pi/2]$ is almost 1;
% \item $L_i$ is a strictly convex function;
% \item The optimal solution $a^*_i$ exists in the range $[l_a, u_a]$ and is unique;
% \item The optimal $a^*_i$ is monotonic increasing with a decreasing intra-class distance (\textit{i.e.}, $\theta_{y_i}$ gets smaller);
% \item The optimal $a^*_i$ is monotonic increasing with a increasing inter-class distance (\textit{i.e.}, $B$ gets  larger);
% \end{enumerate}

\begin{prop1*}
  For $a_i \in [l_a, u_a]$, $L_i$ is a strictly convex function which has a unique optimal solution $a_i^*$.
\end{prop1*}

\begin{prop2*}
   The optimal $a_i^*$ is monotonically increasing as the cosine-distance to its class center decreases and the cos-distances to other classes increase.
  % With fixed $f_i$ and $W_j, j\in\{1,\cdots, n\}, j\neq y_i$, the optimal feature magnitude $a^*_i$ is monotonically decreasing w.r.t. the intra-class distance to class center $W_{y_i}$ increases.
  % With other things fixed, the optimal feature magnitude $a^*_i$ is monotonically decreasing with a decreasing inter-class distance $B$.
\end{prop2*}

% The additive margin $m(a_i)$ is always much set to be much smaller than  $\pi/2$ and t
The property of convergence guarantees the unique optimal solution for $a_i$ as well as the fast convergence. The property of monotonicity states that the feature magnitudes reveal the difficulties for recognition, therefore can be treated as a metric for face qualities.

% Property~\ref{prop:mono} states that feature magnitudes are monotonically decreasing with increasing sample qualities. Consequently, easy samples are pushed close to class centers by the dynamical decision boundaries controlled by  $m(a_i)$.

\subsection{Analysis on Feature Magnitude}
\label{subsec:magnitude}
\vspace{-3pt}

To better understand the effect of the MagFace loss, we conduct experiments on the widely used MS1M-V2~\cite{Deng2018} dataset and investigate for the training examples at convergence the relation between the feature magnitude and their similarity with class center as shown in Fig.~\ref{fig:data_example}.

% compare with two representative losses, Softmax and ArcFace~\cite{Deng2018}.
% to study the feature magnitudes and their connections to face qualities.

\textbf{Softmax.} The classical softmax-based loss underlies the objective of the pioneer work~\cite{taigman2014deepface,sun2014deep} on deep face recognition.
Without explicit constraint on magnitude, the value of the negative loss for each sample is almost independent to its magnitude as observed from Fig.~\ref{fig:data1}.
As pointed in \cite{Ranjan2017,Wang2017}, softmax tends to create a radial feature distribution because softmax loss acts as the soft version of max operator and scaling the feature magnitude does not affect the assignment of its class.
To eliminate this effect, \cite{Ranjan2017,Wang2017} suggest that using normalized feature would benefit the task.

% Experimental results in indicate L2 norm of features learned using Softmax loss is informative of the quality of the face.  We observe similar phenomenon in Fig.~\ref{fig:data1} as feature magnitudes are relatively large when the negative loss is larger than 0.6. Without further constraints, Euclidean margin based losses increases magnitudes for correctly classified samples. Easy samples, which is relatively easy to be correctly recognized, tend to have features with large magnitudes.  This phenomenon can be explained by  the Lemma~2(in the supplementary) that was introduced by~\cite{Wang2017}.

\textbf{ArcFace.} ArcFace % (Eq.~\ref{eq:arcloss})
can be considered as a special case of MagFace % (Eq.~\ref{eq:magloss})
when $m(a_i) = m$ and $g(a_i)=0$.
As shown in Fig.~\ref{fig:data2}, high-quality samples with large similarity $\cos(\theta)$ to class center yield large variation in magnitude.
This evidence echos our motivation on the unstable structure defined by a fixed angular margin in ArcFace for easy samples.
On the other hand, for low-quality samples that are difficult to be recognized ($\cos(\theta)$ is small), the fixed angular margin determines the magnitude needs to be large enough in order to fit inside the feasible region (Fig.~\ref{fig:method0}).
Therefore, there is a decreasing low bound for feature magnitudes \wrt the quality of face as indicated by the dash line in Fig.~\ref{fig:data2}.

% This phenomenon is mainly caused by the margin constraints. For the cosine margin based loss, features from the same class are constrained in a fan-shaped area as shown in Fig.~\ref{fig:method0}. Considering that features from low-quality faces are of higher uncertainty (variance)~\cite{Chang2020,Kolf1,Shi2019}, we use radius of circles to represent the qualities.  For features with high uncertainty (\textit{e.g.}, circle 3),  too close to the origin point $o$ can reach to infeasible areas. On the other hand, the magnitudes can be infinity while staying feasible.

\textbf{MagFace.} In contrast to ArcFace, our MagFace optimizes the feature with adaptive margin and regularization based on its magnitude.
Under this loss, it is clear to observe from Fig.~\ref{fig:data3} that there is a strong correlation between the feature magnitudes and their cosine similarities with class center.
Those examples at the upper-right corner are the most high-quality ones.
As the magnitude becomes smaller, the examples are more deviated from the class center.
This distribution strongly supports the fact that the feature magnitude learned by MagFace is a good metric for face quality.

\section{Experiments}
\vspace{-1pt}

\begin{figure*}[htb!]
  \centering
  \subfloat[][mean: 22.84 \\ range: (-$\infty$, 24) \\ \# of faces: 3692 ]{\includegraphics[width=0.12\textwidth]{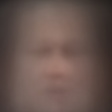}}
  \subfloat[][mean: 25.13 \\ range: [24, 26) \\ \# of faces: 9955 ]{\includegraphics[width=0.12\textwidth]{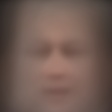}}
  \subfloat[][mean: 27.03 \\ range: [26, 28) \\ \# of faces: 15459]{\includegraphics[width=0.12\textwidth]{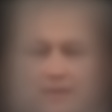}}
  \subfloat[][mean: 29.03 \\ range: [28, 30) \\ \# of faces: 17565]{\includegraphics[width=0.12\textwidth]{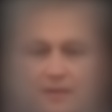}}
  \subfloat[][mean: 31.01 \\ range: [30, 32) \\ \# of faces: 20627]{\includegraphics[width=0.12\textwidth]{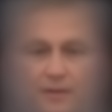}}
  \subfloat[][mean: 32.99 \\ range: [32, 34) \\ \# of faces: 19743]{\includegraphics[width=0.12\textwidth]{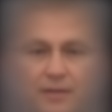}}
  \subfloat[][mean: 34.80 \\ range: [34, 36) \\ \# of faces: 11238]{\includegraphics[width=0.12\textwidth]{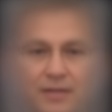}}
  \subfloat[][mean: 36.55 \\ range: [36, $\infty$) \\ \# of faces: 1721]{\includegraphics[width=0.12\textwidth]{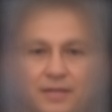}}
  \vspace{4pt}
  \caption{Visualization of the mean faces of 100k images sampled from the IJB-C dataset.
    Each mean face corresponds to a group of faces based on the magnitude level of the features learned by MagFace.
  }
  \label{fig:quality}
\end{figure*}

% In our experiments, we define function $m(a_i)$ as a linear function defined on $[l_a, u_a]$ with $m(l_a) = l_m, m(u_a) = u_m$ and  $g(a_i) = \frac{1}{a_i} + \frac{1}{u_a^2}a_i$. Therefore, we have
% \begin{equation}
%   \resizebox{0.25\textwidth}{!}{$
%     \begin{aligned}
%      & m(a_i) = \frac{u_m-l_m}{u_a-l_a}(a_i-l_a) + l_m \\
%       &\lambda_g \geq \frac{sk}{-g'(l_a)} =  \frac{su_a^2l_a^2}{(u_a^2-l_a^2)} \frac{u_m-l_m}{u_a-l_a}
%     \end{aligned}
%     $}
%   \label{eq:mg}
% \end{equation}

In this section, we examine the proposed MagFace on three important face tasks: face recognition, quality assessment and face clustering. % The extensive experimental results show the superiority of our MagFace.
Sec.~C in the supplementary presents the ablation study on relationships between margin distributions and recognition performances.

% In the end, our mean margin as well as other hyperparameters are all aligned to ArcFace. In the end,  we have 3 hyperparameters: $l_m, u_m, \lambda_g$.

\subsection{Face Recognition} \label{sec:exp_fr}
\vspace{-3pt}

\noindent
\textbf{Datasets.} The original MS-Celeb-1M dataset~\cite{guo2016ms} contains about 10 million images of 100k identities.
However, it consists of a great many noisy face images.
Instead, we employ MS1M-V2~\cite{Deng2018} (5.8M images, 85k identities) as our training dataset.
For evaluation, we adopt LFW~\cite{huang2008labeled}, CFP-FP~\cite{sengupta2016frontal}, AgeDB-30~\cite{moschoglou2017agedb}, CALFW~\cite{zheng2017cross}, CPLFW~\cite{zheng2018cross}, IJB-B~\cite{whitelam2017iarpa} and IJB-C~\cite{maze2018iarpa} as the benchmarks.
All the images are aligned to $112\times 112$ by following the setting in ArcFace.

\noindent
\textbf{Baselines.} We re-implement state-of-the-art baselines including Softmax, SV-AM-Softmax~\cite{Wang2018_sv}, SphereFace~\cite{Liu2017}, CosFace~\cite{Wang2018}, ArcFace~\cite{Deng2018}.
ResNet100 is equipped as the backbone.
We use the recommended hyperparameters for each model, \eg, $s=64$, $m=0.5$ for ArcFace. % and (30, 0.35) for  CosFace~\cite{Wang2018}.

\noindent
\textbf{Training.} We train models on 8 1080Tis by stochastic gradient descent.
The learning rate is initialized from 0.1 and divided by 10 at 10, 18, 22 epochs, and we stop the training at the 25th epoch.
The weight decay is set to 5e-4 and the momentum is 0.9.
We only augment training samples by random horizontal flip.
For MagFace, we fix the upper bound and lower bound of the magnitude as $l_a=10, u_a=110$.
% $(l_m, u_m)$ are set to be around $(0.4, 0.8)$ and the mean margin is controlled to be $0.5$ by changing $\lambda_g$.
$m(a_i)$ is chosen to be a linear function and $g(a_i)$ as a hyperbola.
For detailed definition of $m(a_i)$, $g(a_i)$ and $\lambda_g$, please refer to Sec.~B2 in the supplementary.
In the end, our mean margin as well as other hyperparameters are all consistent with ArcFace.

\noindent
\textbf{Test.} During testing, cosine distance is used as metric on comparing 512-D features.
For evaluations on IJB-B/C, one identity can have multiple images.
The common way to represent for an identity is to sum up the normalized feature $f_i^{norm} = \frac{f_i}{\|f_i\|}$ of each image and then normalize the embedding for comparisons, \ie, $f=\frac{\sum_i f_i^{norm}}{\| \sum_i f_i^{norm}  \|}$.
One benefit of MagFace is that we can assign quality-aware weight $\|f_i\|$ to each normalized feature $f_i^{norm}$.
Therefore, we further evaluate ``MagFace+'' in Tab.~\ref{table:fr_ijb} by computing the identity embedding as $f_+=\frac{\sum_i f_i}{\| \sum_i f_i  \|}$.

\setlength{\tabcolsep}{4pt}
\begin{table}
  \begin{center}
    % \resizebox{0.5\textwidth}{!}{%
    \footnotesize{
      \begin{tabular}{l|ccccc}
        \hline
        Method &  LFW & CFP-FP & AgeDB-30 & CALFW & CPLFW  \\
        \hline \hline
        Softmax & 99.70 & 98.20 & 97.72 & 95.65 & 92.02 \\
        SV-AM-Softmax~\cite{Wang2018_sv}& 99.50 &  95.10 &  95.68 & 94.38 & 89.48 \\
        SphereFace~\cite{Liu2017} & 99.67 &  96.84 &  97.05 &95.58 & 91.27\\
        CosFace~\cite{Wang2018} & 99.78  &  98.26  &   \textbf{98.17} & \textbf{96.18} & 92.18\\
        ArcFace~\cite{Deng2018} & 99.81 & 98.40 & 98.05 & 95.96 & 92.72\\
        \rowcolor{Gray}
        MagFace & \textbf{99.83} &  \textbf{98.46} &  \textbf{98.17} & 96.15 & \textbf{92.87}  \\
        \hline
      \end{tabular}
    }
  \end{center}
  \caption{Verification accuracy (\%) on easy benchmarks.}      \label{table:fr_faces}
\end{table}
\setlength{\tabcolsep}{1.4pt}

% LFW dataset contains 13233 web-collected images from 5749 different identities, with large variations in pose, expression and illuminations.
% CFP dataset consists of 500 subjects, each with 10 frontal and 4 profile images. The frontal-profile (FP) face verification, each having 10 folders with 350 same-person pairs and 350 different-person pairs.
% AgeDB-30 dataset is an in-the-wild dataset with large variations in pose, expression, illuminations, and age. AgeDB-30 contains 12, 240 images of 440 distinct subjects, such as actors, actresses, writers, scientists, and politi- cians. Each image is annotated with respect to the identity, age and gender attribute. The minimum and maximum ages are 3 and 101, respectively. The average age range for each subject is 49 years. There are four groups of test data withdifferent year gaps (5 years, 10 years, 20 years and 30 years, respectively) [10]. Each group has ten split of face images, and each split contains 300 positive examples and 300 neg- ative examples. The face verification evaluation metric is the same as LFW. In this paper, we only use the most chal- lenging subset, AgeDB-30-30, to report the performance.

\noindent
\textbf{Results on LFW, CFP-FP, AgeDB-30, CALFW and CPLFW.}
We directly use the aligned images and protocols adopted by ArcFace~\cite{Deng2018} and present our results in Tab.~\ref{table:fr_faces}.
We note that performances are almost saturated.
Compared to CosFace which is the second best baseline, ArcFace achieves $0.03\%, 0.14\%, 0.54\%$ improvement on LFW, CFP-FP and CPLFW, while drops $0.12\%, 0.22\%$ on AgeDB-30 and CALFW.
MagFace obtains the overall best results and surpasses ArcFace by $0.02\%$, $0.06\%$, $0.12\%$, $0.19\%$ and $0.15\%$ on five benchmarks respectively.

\setlength{\tabcolsep}{4pt}
\begin{table}
  \begin{center}
    % \resizebox{0.5\textwidth}{!}{%
    \footnotesize{
      \begin{tabular}{lcccccc}
        \hline
        % \multirow{4}{c}{Method}
        Method
        & \multicolumn{3}{c}{IJB-B (TAR@FAR)} & \multicolumn{3}{c}{IJB-C (TAR@FAR)} \\
        \cline{2-7}
        & 1e-6 &  1e-5 & 1e-4 & 1e-6 & 1e-5 & 1e-4 \\
        \hline
        % \rowcolor{Gray1}
        VGGFace2*~\cite{cao2018vggface2} & - & 67.10 & 80.00 & - & 74.70 & 84.00 \\
        % \rowcolor{Gray1}
        CenterFace*~\cite{wen2016discriminative} & - & - & - & - & 78.10 & 85.30 \\
        % \rowcolor{Gray1}
        CircleLoss*~\cite{Sun2020} &-&-&-&-&89.60&93.95\\
        % \rowcolor{Gray1}
        ArcFace*~\cite{Deng2018} & - & - & 94.20 & - & - & 95.60 \\
        \hline
        Softmax & \textbf{46.73}	&75.17&90.06	&64.07	&83.68	&92.40\\
        SV-AM-Softmax~\cite{Wang2018_sv}& 29.81 & 69.25 & 84.79 & 63.45 & 80.30 & 88.34 \\
        SphereFace~\cite{Liu2017} &  39.40&  73.58 &  89.19 &  68.86 &  83.33 &  91.77 \\
        CosFace~\cite{Wang2018} &  40.41&  89.25 &  94.01 &   87.96 &  92.68  &  95.56\\
        ArcFace~\cite{Deng2018} & 38.68 & 88.50 & 94.09 & 85.65 & 92.69 & 95.74 \\
        \rowcolor{Gray}
        MagFace & 40.91 & 89.88 & 94.33 & 89.26 & 93.67 & 95.81 \\
        \rowcolor{Gray}
        MagFace+ & 42.32 & \textbf{90.36} & \textbf{94.51} & \textbf{90.24} & \textbf{94.08} & \textbf{95.97} \\
        \hline
      \end{tabular}
    }
  \end{center}
  \caption{Verification accuracy (\%) on difficult benchmarks. ``*'' indicates the result quoted from the original paper.} \label{table:fr_ijb}
\end{table}
\setlength{\tabcolsep}{1.4pt}

\noindent
\textbf{Results on IJB-B/IJB-C.}
The IJB-B dataset contains 1,845 subjects with 21.8K still images and 55K frames from 7,011 videos.
As the extension of IJB-B, the IJB-C dataset covers about 3,500 identities with a total of 31,334 images and 117,542 unconstrained video frames.
In the 1:1 verification, the number of positive/negative matches are 10k/8M in IJB-B and 19k/15M in IJB-C.
We report the TARs at FAR=1e-6, 1e-5 and 1e-4 as shown in Tab.~\ref{table:fr_ijb}.

Our implemented ArcFace is on par with the original paper, \eg, our TARs at FAR=1e-4 differ from the authors by $-0.11\%$ and $+0.14\%$ on IJB-B and IJB-C respectively.
Compared to baselines, our MagFace remains the top at all FAR criteria except for FAR=1e-6 on IJB-B as the TAR is very sensitive to the noise when the number of FP is tiny.
Compared to CosFace, MagFace gains $0.50\%, 0.63\%, 0.32\%$ on IJB-B at TAR@FAR=1e-6, 1e-5, 1e-4 and $1.30\%, 0.99\%, 0.25\%$ on IJB-C.
Compared to ArcFace, improvements are of $2.23\%, 1.38\%, 0.24\%$ on IJB-B and $3.61\%, 0.98\%, 0.07\%$ on IJB-C respectively.
This result demonstrates the superiority of MagFace on more challenging benchmarks.
It is worth to mention that when multiple images existed for one identity, the average embedding can be further improved by aggregating features weighted by magnitudes.
For instance, MagFace+ outperforms MagFace by $1.41\%/0.98\%$ at FAR=1e-6, $0.48\%/0.41\%$ at FAR=1e-5 and $0.18\%/0.16\%$ at FAR=1e-4.

\subsection{Face Quality Assessment}\label{sec:exp_fq}
\vspace{-3pt}

\begin{figure}[htb!]
  \centering
  \includegraphics[width=0.4\textwidth]{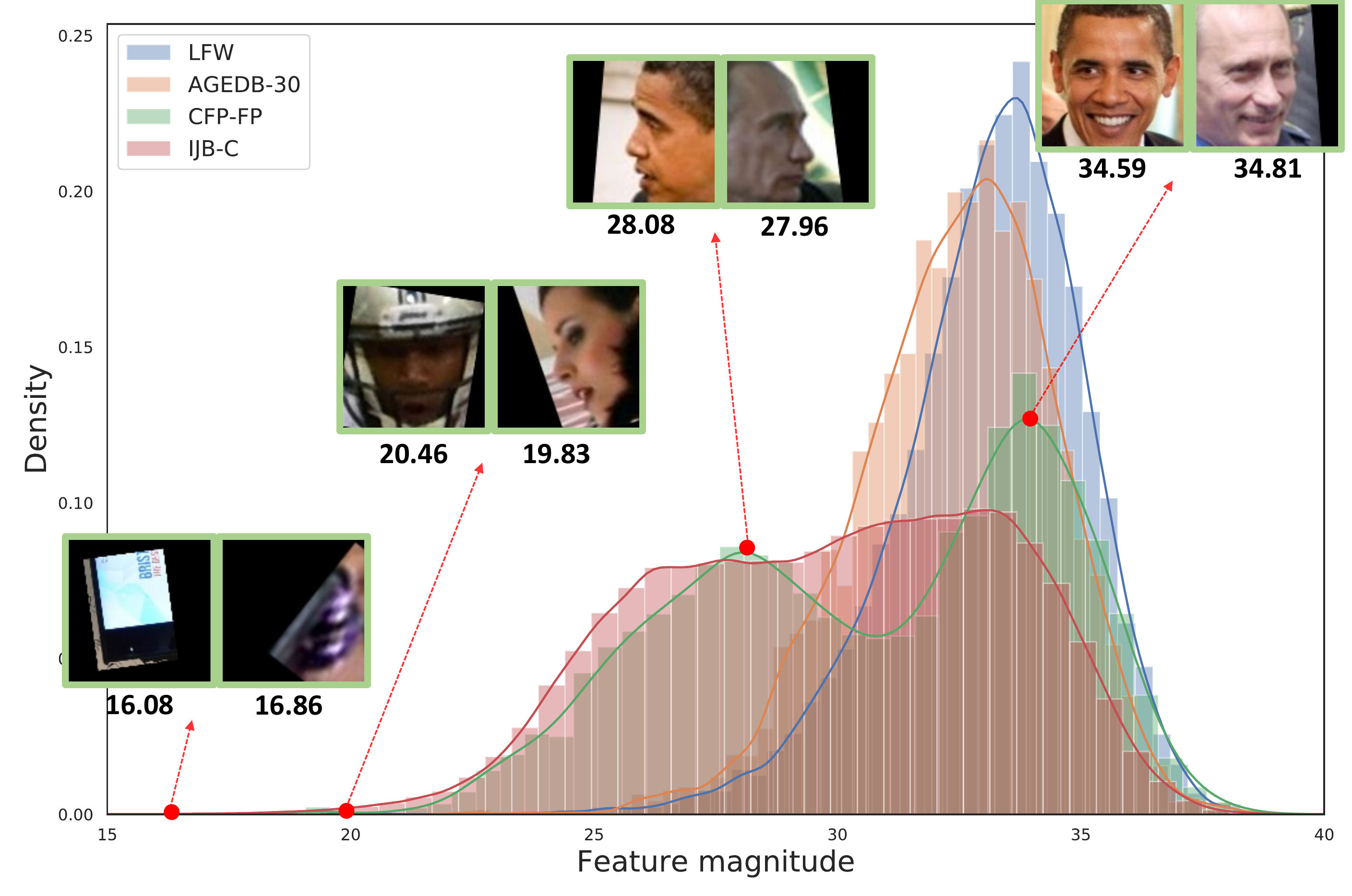}
  \caption{Distributions of magnitudes on different datasets. % \textbf{Best viewed in color.}
  }
  \label{fig:dist_qlt}
\end{figure}

\begin{figure*}[!htb]
  \centering
  \subfloat[LFW - ArcFace]{\includegraphics[trim=20 5 50 30,clip, width=0.45\textwidth]{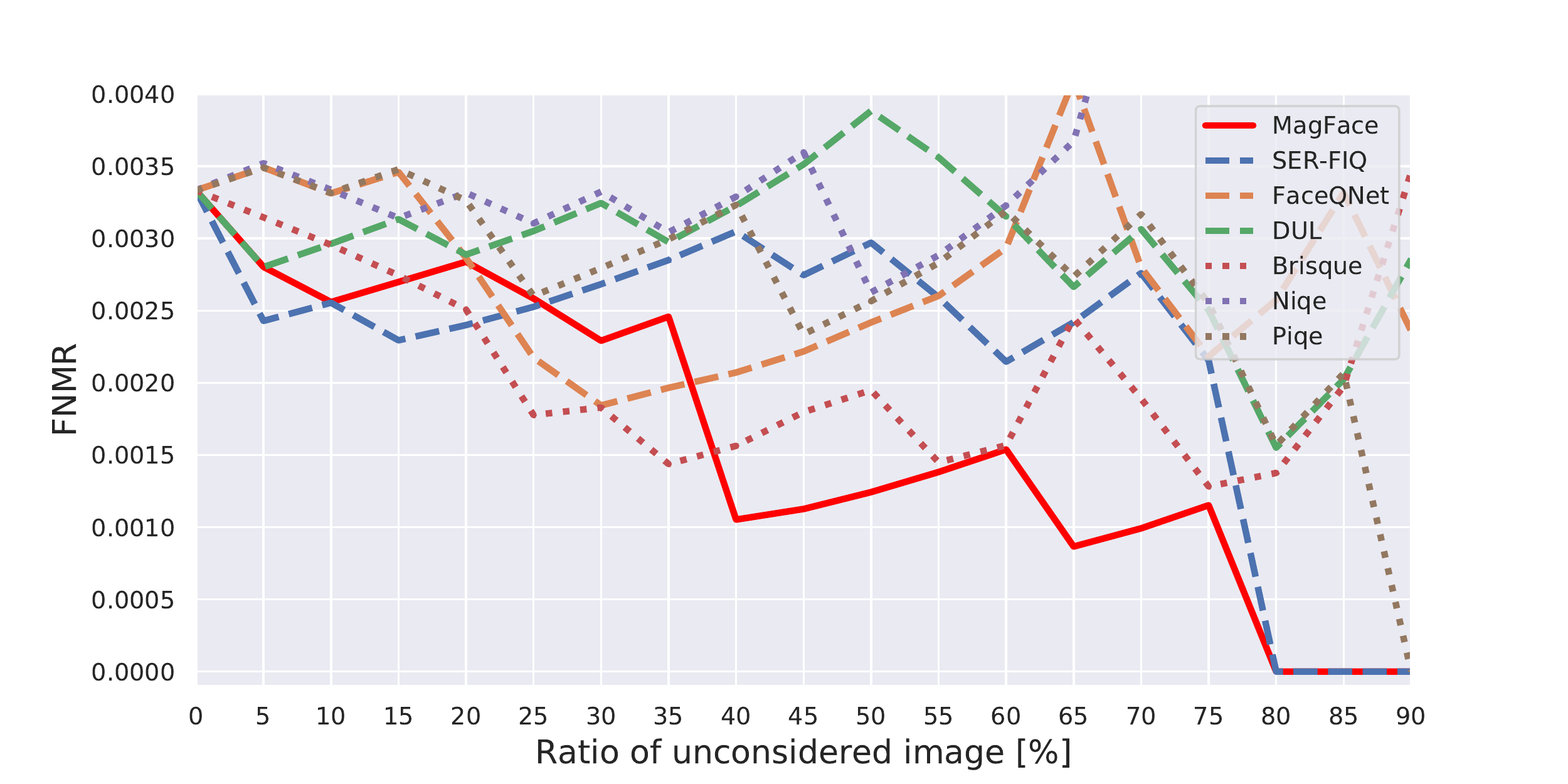}}   \subfloat[LFW - MagFace]{\includegraphics[trim=20 5 50 30,clip,width=0.45\textwidth]{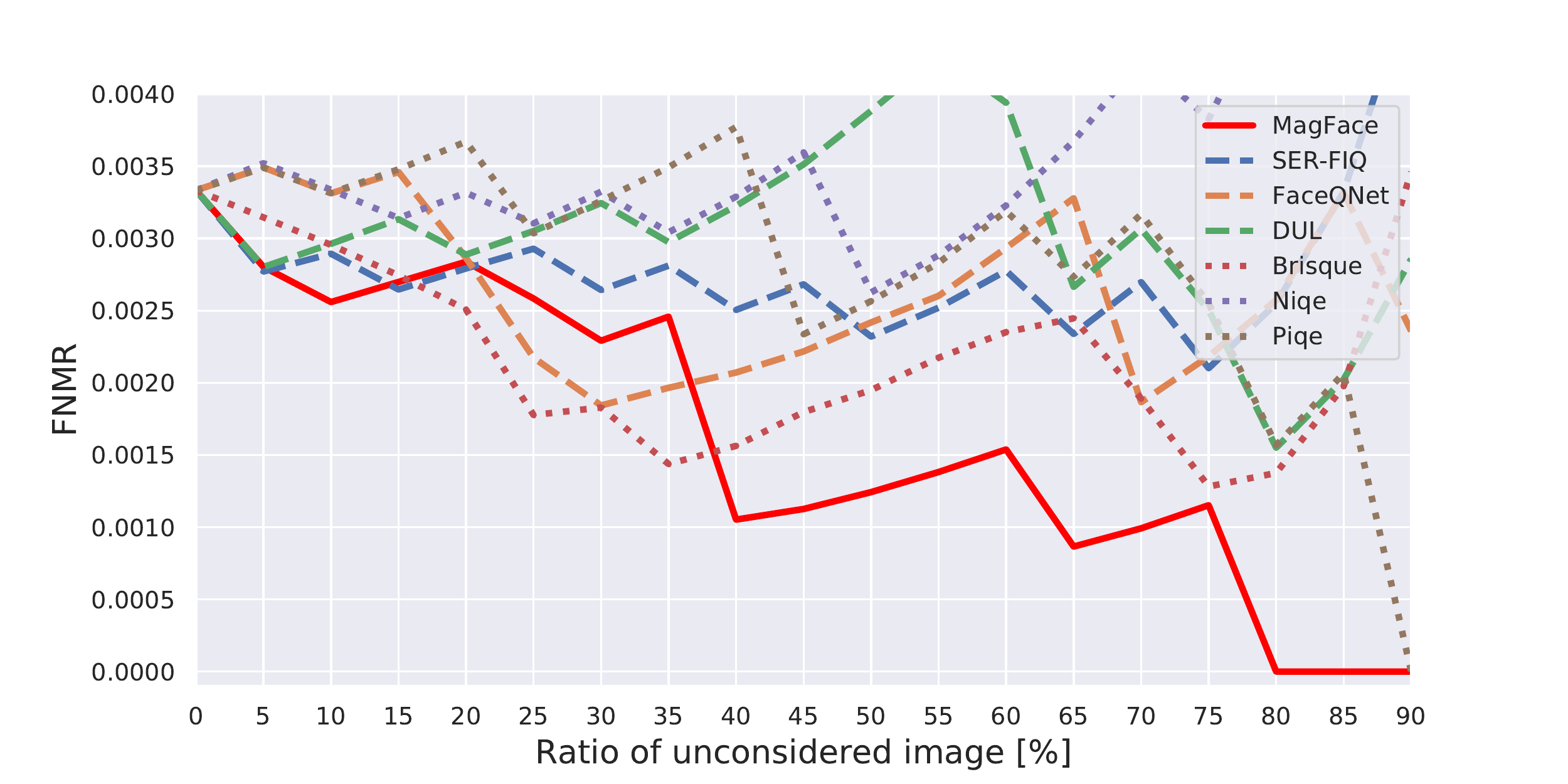}}\\
  \vspace{-12pt}
  \subfloat[CFP-FP - ArcFace]{\includegraphics[trim=20 5 50 30,clip, width=0.45\textwidth]{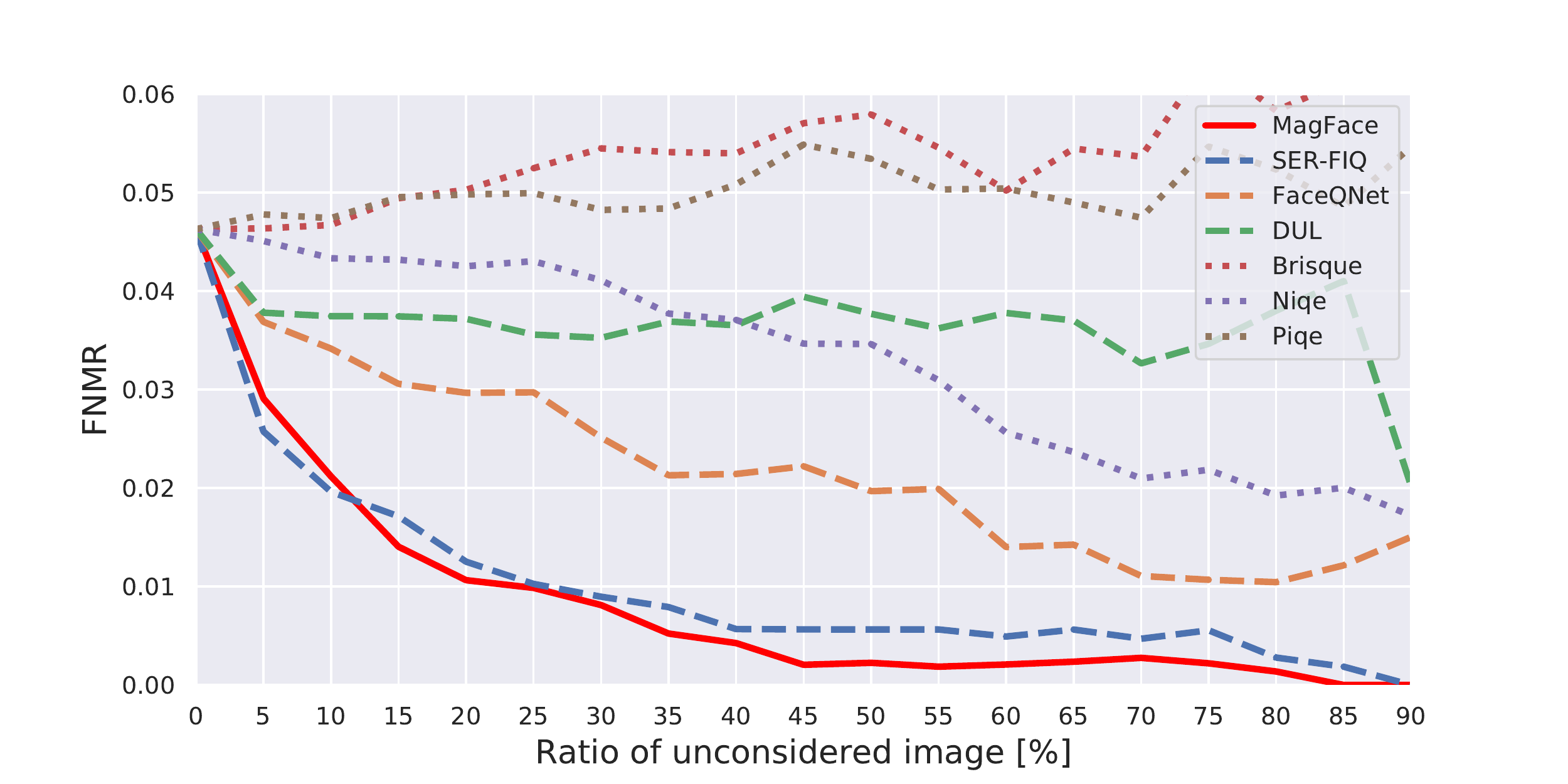}}  \subfloat[CFP-FP - MagFace]{\includegraphics[trim=20 5 50 30,clip, width=0.45\textwidth]{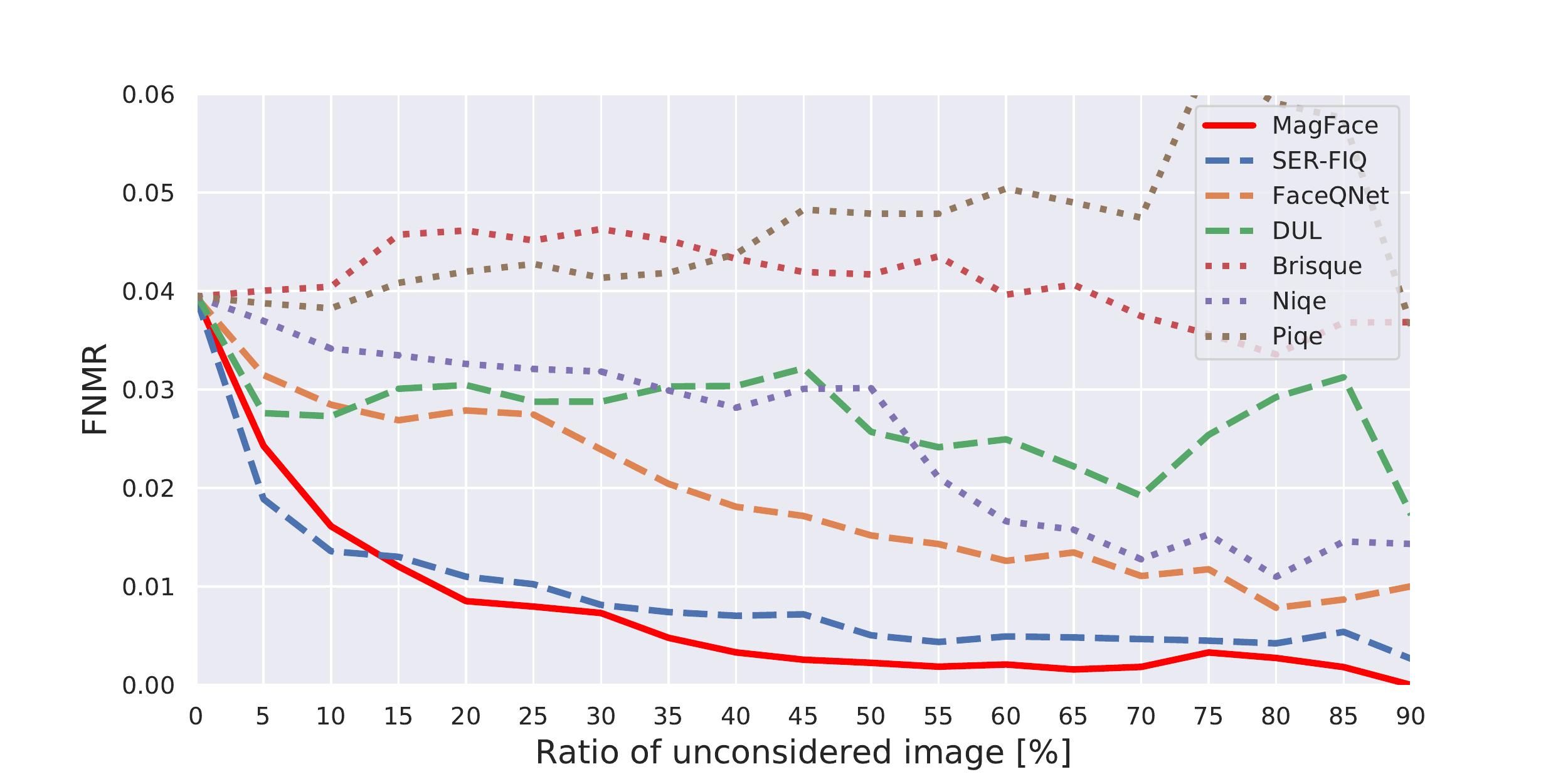}}\\
  \vspace{-12pt}
  \subfloat[AgeDB-30 - ArcFace]{\includegraphics[trim=20 5 50 30,clip, width=0.45\textwidth]{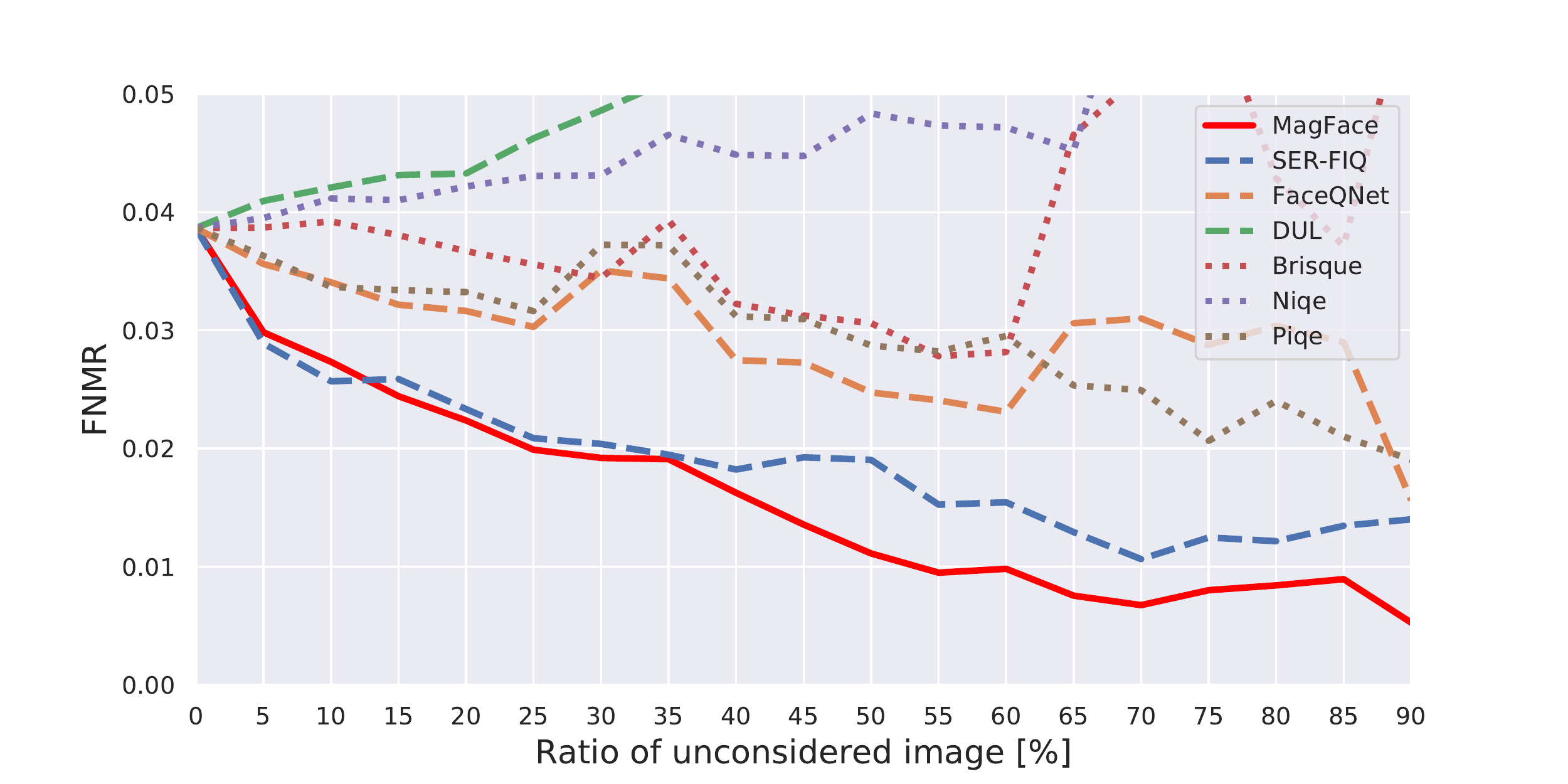}} \subfloat[AgeDB-30 - MagFace]{\includegraphics[trim=20 5 50 30,clip, width=0.45\textwidth]{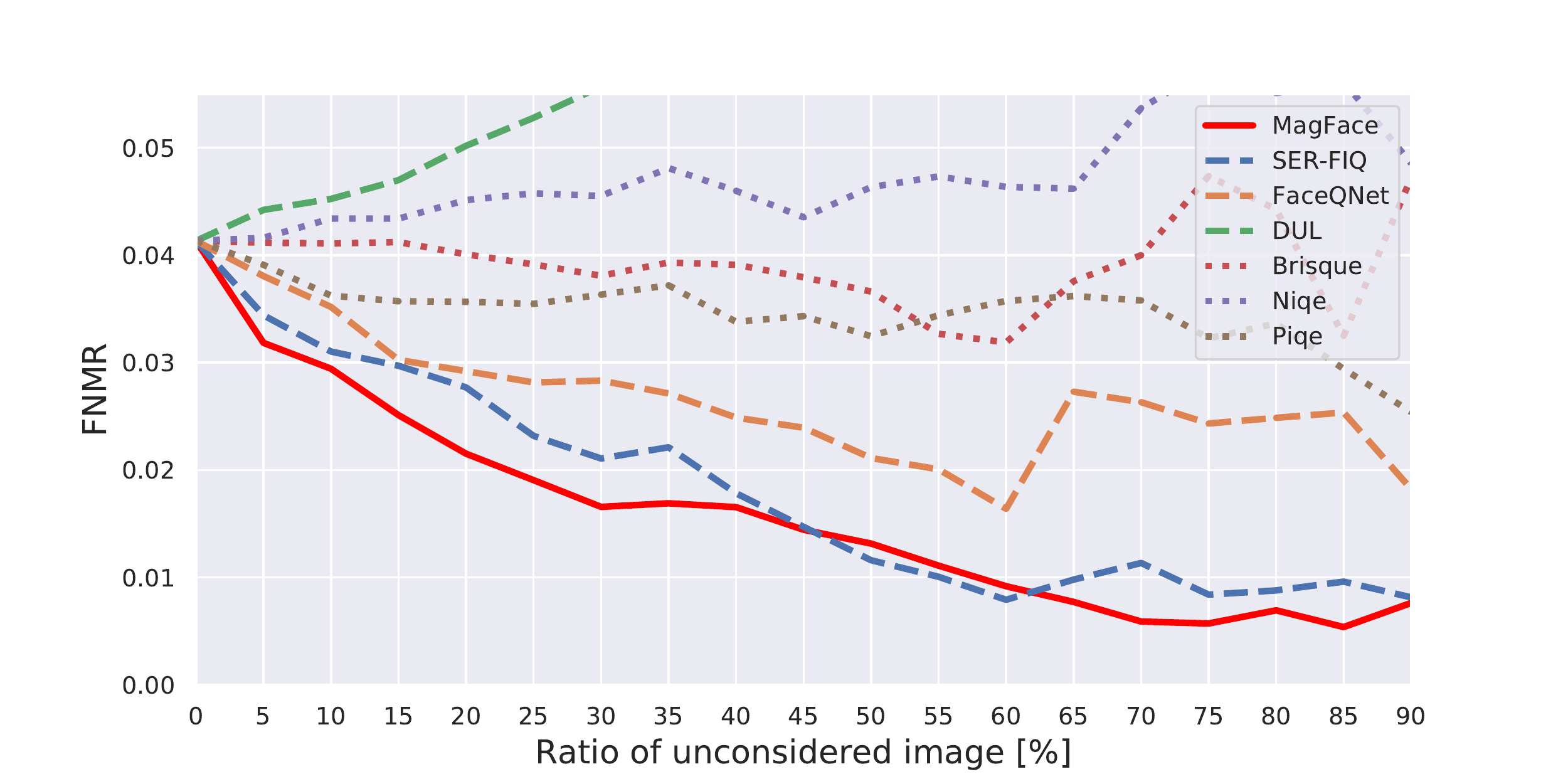}}\\
  \caption{Face verification performance for the predicted face quality values with two evaluation models (ArcFace and MagFace).
    The curves show the effectiveness of rejecting low-quality face images in terms of false non-match rate (FNMR). \textbf{Best viewed in color.}}
  \label{fig:qlt_fnmr}
\end{figure*}

In this part, we investigate the qualitative and quantitative performance of the pre-trained MagFace model mentioned in Tab.~\ref{table:fr_ijb} for quality assessment.

% We use the pre-trained model provided by the authors of ArcFace to evaluate the performances of our method and five other quality assessment methods. To further demonstrate that our MagFace ``knows which images are suitable for it to recognize``, we also test the results of our  MagFace  under different quality-assessment methods. The experimental results show our the efficiency of MagFace in the task of quality assessment.

\noindent
\textbf{Visualization of the mean face.}
We first sample 100k images form IJB-C database and divide them into 8 groups based on feature magnitudes.
We visualize the mean faces of each group in Fig.~\ref{fig:quality}.
It can be seen that when magnitude increases, the corresponding mean face reveals more details.
This is because high-quality faces are inclined to be more frontal and distinctive.
This implies the magnitude of MagFace feature is a good quality indicator.

\noindent
\textbf{Sample distribution of datasets.}
Fig.~\ref{fig:dist_qlt} plots the sample histograms of different benchmarks with respect to MagFace magnitudes.
We observe that LFW is the least noisy one where most samples are of large magnitudes.
Due to the larger age variation, the distribution of AGEDB-30 slightly shifts left compared to LFW.
For CFP-FP, there are two peaks at the magnitude around 28 and 34, corresponding to the frontal and profile faces respectively.
Given the large variations in face qualities, we can conclude IJB-C is much more challenging than other benchmarks.
For images (more examples can be found in the supplementary) with magnitudes $a \simeq 15$, there are no faces or very noisy faces to observe.
When feature magnitudes increase from 20 to 40, there is a clear trend that the face changes from profile, blurred and occluded, to more frontal and distinctive.
Overall, this figure convinces us that MagFace is an effective tool to rank face images according to their qualities.

\noindent
\textbf{Baselines.}
We choose six baselines of three types for quantitative quality evaluation.
Brisque~\cite{sun2015no}, Niqe~\cite{Mittal2013} and Piqe~\cite{venkatanath2015blind} are image-based quality metrics.
FaceQNet~\cite{Hernandez-ortega} and SER-FIQ~\cite{Kolf1} are face-based ones.
For FaceQNet, we adopt the released models by the authors.
For SER-FIQ, we use the ``same model'' version which yields the best performance in the paper.
Following the authors' setting, we set $m=100$ to forward each image 100 times with drop-out active in inference.
As a related work, we re-implement the recent DUL~\cite{Chang2020} method that can estimate uncertainty along with the face feature.

\noindent
\textbf{Evaluation metric.}
Following previous work~\cite{Grother2007, Kolf1, Best-Rowden2017}, we evaluate the quality assessment on LFW/CFP-FP/AgeDB via the error-versus-reject curves, where images with the lowest predicted qualities are unconsidered and error rates are calculated on the remaining images.
Error-versus-reject curve indicates good quality estimation when the verification error decreases consistently while increasing the ratio of unconsidered images.
To compute the feature for verification, we adopt the ArcFace* as well as our MagFace models in Tab.~\ref{table:fr_ijb}.
% To further demonstrate that our MagFace ``knows which images are suitable for it to recognize``, we also test the results of our  MagFace  under different quality-assessment methods.
% The experimental results show our the efficiency of MagFace in the task of quality assessment.

% and as equal error rate (EER).The EER equals the FMR at the threshold where FMR = 1−FNMR and is well known as a single-value indi- cator of the verification performance. These error rates are specified for biometric verification evaluation in the interna- tional standard [24].

\noindent
\textbf{Results on face verification.}
Fig.~\ref{fig:qlt_fnmr} shows the error-versus-reject curves of different quality methods in terms of false non-match rate (FNMR) reported at false match rate (FMR) threshold of 0.001.
Overall, we have two high-level observations.
1) The curves on CFP-FP and AgeDB-30 are much more smooth than the ones obtained on LFW.
This is because CFP-FP and AgeDB-30 consist of faces with larger variations in pose and age.
Effectively dropping low-quality faces can benefit the verification performance more on these two benchmarks.
2) No matter computing the feature from ArcFace (left column) or MagFace (right column), the curves corresponding to MagFace magnitude are consistently the lowest ones across different benchmarks.
This indicates that the performance of MagFace magnitude as quality generalizes well across datasets as well as face features.
We then analyze the quality performance of each type of methods.
1) The image-based quality metrics (Brisque~\cite{sun2015no}, Niqe~\cite{Mittal2013}, Piqe~\cite{venkatanath2015blind}) lead to relatively higher errors in most cases as the image quality alone is not suitable for generalized face quality estimation.
Factors of the face (such as pose, occlusions, and expressions) and model biases are not covered by these algorithms and might play an important role for face quality assessment.
2) The face-based methods (FaceQNet~\cite{Hernandez-ortega} and SER-FIQ~\cite{Kolf1}) outperforms other baselines in most cases.
In particular, SER-FIQ is more effective than FaceQNet in terms of the verification error rates.
This is due to the fact that SER-FIQ is built on top of the deployed recognition model so that its prediction is more suitable for the verification task.
However, SEQ-FIQ takes a quadratic computational cost \wrt the number of sub-networks $m$ randomly sampled using dropout.
In contrary, the neglectable overhead of computing magnitude makes the proposed MagFace more practical in many real-time scenarios.
Moreover, the training of MagFace does not require explicit labeling of face quality, which is not only time consuming but also error-prone to obtain.
3) At last, the uncertainty method (DUL) performs well on CFP-FP but yields more verification errors on AgeDB-30 when the proportion of unconsidered images is increased.
This may indicate that the Gaussian assumption of data variance in DUL is over-simplified such that the model cannot generalize well to different kinds of quality factors.

% On all benchmarks, MagFace outperforms other baseline methods, which show the generosity of qualities estimated by out method.

% Besides in an unsupervised manner, we emphasize another advantage of our MagFace is the inference speed as  qualities are directly estimated by feature magnitudes. In contrast, FaceQNet needs to feed images into an additional network to estimate the face qualities, besides the recognition model.
% For SER-FIQ, features are extracted  100 times for each image, which is much slower than our method.

\subsection{Face Clustering}\label{sec:exp_fc}
\vspace{-3pt}

In this section, we conduct experiments on face clustering to further investigate the structure of feature representations learned by MagFace.

\begin{figure}[!htb]
  \centering
  \includegraphics[trim=30 0 30 40,clip, width=0.4\textwidth]{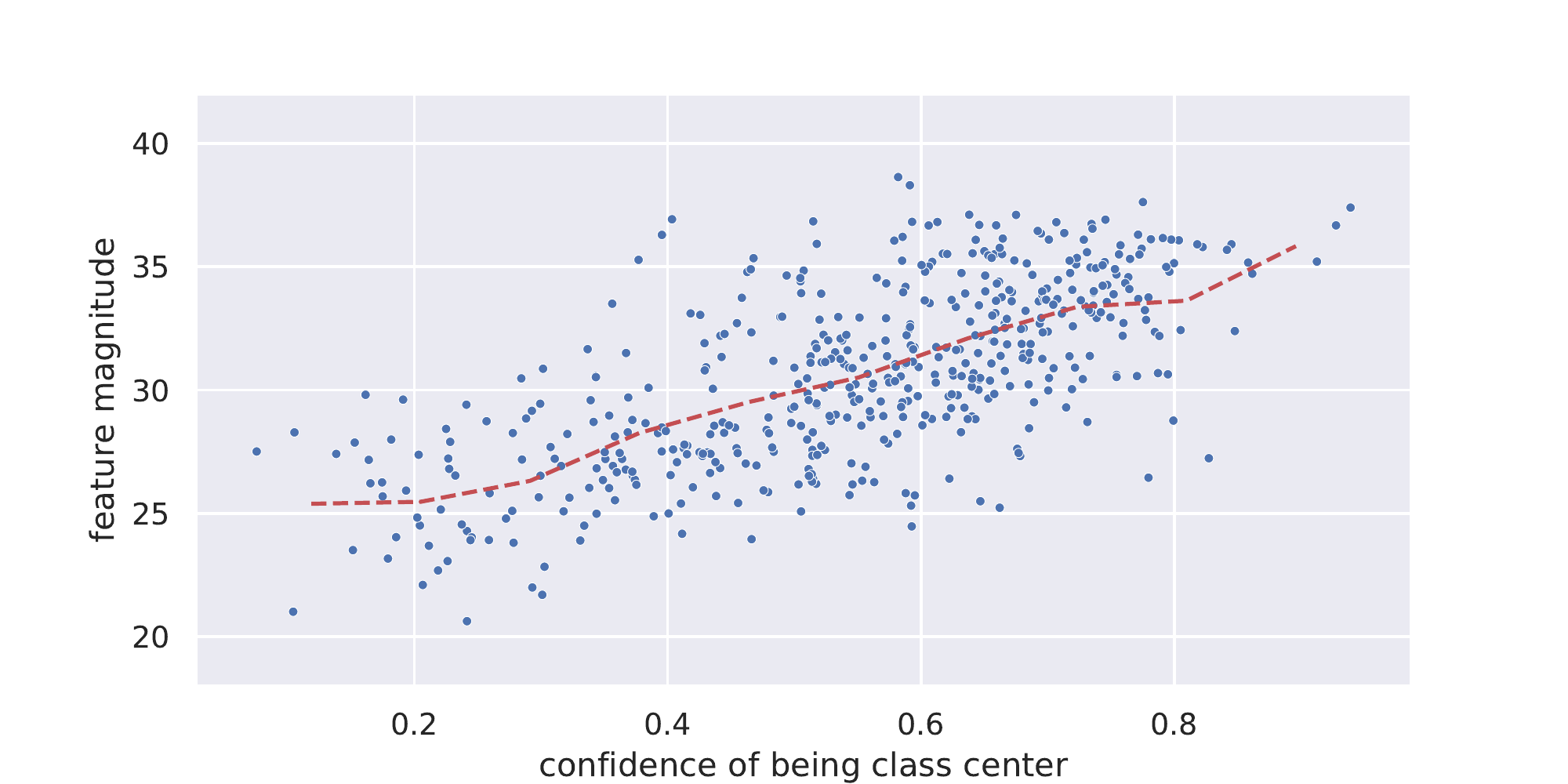}
  \caption{Visualization of MagFace magnitudes of 500 samples from IJB-B-1845 \wrt their confidences of being class centers.}
  \label{fig:cluster_vis}
\end{figure}

\noindent
\textbf{Baselines.}
We compare the performances of MagFace and ArcFace by integrating their features with various clustering methods.
For fair comparisons, we constrain hyperparameters of the two models to be consistent (\eg, s=64, mean margin 0.5) during training.
% By this means, the only difference of features in the angular direction is the distribution.
Four clustering methods are used in the evaluation: K-means~\cite{lloyd1982least}, AHC~\cite{sarle1990algorithms}, DBSCAN~\cite{ester1996density} and L-GCN~\cite{wang2019linkage}.
% We tune the best hyperparameters e.g. $\varepsilon, n$ in DBSCAN, and report the best results.
For non-deterministic algorithms (K-means and AHC), we report the average results from 10 runs.
% and set $\varepsilon=0.8, n=1$ for DBSCAN.
For L-GCN, we train the model on CASIA-WebFace~\cite{yi2014learning} (0.5M images from 10k individuals) and follow the recommended settings in the paper~\cite{wang2019linkage}.
% , setting the number of 1-hop nodes and 2-hop nodes as $k_1$ = 200 and $k_2$ = 10.
% Emphatically, the same hyperparameters are utilized in each clustering method for Arcface and Magface.

\noindent
\textbf{Benchmarks.} We adopt the IJB-B~\cite{whitelam2017iarpa} dataset as the benchmark as it contains a clustering protocol of seven sub-tasks varying in the number of ground truth identities.
Following~\cite{wang2019linkage}, we evaluate on three largest sub-tasks where the numbers of identities are 512, 1,024 and 1,845, and the numbers of samples are 18,171, 36,575 and 68,195, respectively.
Normalized mutual information (NMI) and BCubed F-measure~\cite{amigo2009comparison} are employed as the evaluation metrics.

% We set $\varepsilon=0.8, n=1$ in DBSCAN. In L-GCN evaluation, we choose $k_1$ = 80 and $k_2$ = 5 for the IJB-B dataset as same as~\cite{wang2019linkage}. For non-deterministic algorithms like K-means and AHC we report the average result from 10 runs.

\setlength{\tabcolsep}{4pt}
\begin{table}
  \begin{center}
    \resizebox{0.47\textwidth}{!}{%
      \begin{tabular}{lccccccc}
        \hline
        Method & Net & \multicolumn{2}{c}{IJB-B-512} & \multicolumn{2}{c}{IJB-B-1024} & \multicolumn{2}{c}{IJB-B-1845}  \\
        \cline{3-8}
               & & F & NMI & F & NMI & F & NMI \\
        \hline
        \hline
        K-means~\cite{lloyd1982least} & ArcFace & 66.70 & 88.83 & 66.82 & 89.48 &  66.93 & 89.88 \\
        \rowcolor{Gray1}   \cellcolor{white} &  MagFace & \textbf{66.75} & \textbf{88.86} & \textbf{67.33}& \textbf{89.62} &  \textbf{67.06} & \textbf{89.96} \\
        AHC~\cite{sarle1990algorithms} & ArcFace & 69.72 & 89.61 & 70.47 & 90.54 & 70.66 & 90.90 \\
        \rowcolor{Gray1}   \cellcolor{white}      & MagFace & \textbf{70.24} & \textbf{89.99} & \textbf{70.68} & \textbf{90.67} & \textbf{70.98}&  \textbf{91.06} \\
        DBSCAN~\cite{ester1996density} & ArcFace & 72.72 & 90.42 & 72.50 & 91.15 &  73.89 & 91.96 \\
        \rowcolor{Gray1}   \cellcolor{white}        & MagFace & \textbf{73.13}& \textbf{90.61} & \textbf{72.68}& \textbf{91.30} &  \textbf{74.26}& \textbf{92.13} \\
        L-GCN~\cite{wang2019linkage} & ArcFace & 84.92 & 93.72 & 83.50 & 93.78 & 80.35 & 92.30 \\
        \rowcolor{Gray1}   \cellcolor{white}  & MagFace & \textbf{85.27} & \textbf{93.83} & \textbf{83.79} & \textbf{94.10} &  \textbf{81.58} &  \textbf{92.79} \\
        \hline
      \end{tabular}
    }
  \end{center}
  \caption{F-score (\%) and NMI (\%) on clustering benchmarks.} \label{table:clustering1}
\end{table}
\setlength{\tabcolsep}{1.4pt}

\noindent
\textbf{Results.}
Tab.~\ref{table:clustering1} summarizes the clustering results.
We can observe that with stronger clustering methods from K-means to L-GCN, the overall clustering performance can be improved.
For any combination of clustering and protocol, MagFace always achieves better performance than ArcFace in terms of both F-score and NMI metrics.
This consistent superiority demonstrates the MagFace feature is more suitable for clustering.
Notice that we keep the same hyperparameters for clustering.
The improvement of using MagFace must come from its better within-class feature distribution, where the high-quality samples around the class center are more likely to be separated across different classes.

We further explore the relationship between feature magnitudes and the confidences of being class centers.
Following the idea mentioned in \cite{yang2020learning}, the confidence of being a class center for each sample is estimated based on its neighbor structure defined by face features.
The samples with dense and pure local connection have high confidence, while those with sparse connections or residing in the boundary among several clusters have low confidence.
From Fig.~\ref{fig:cluster_vis}, it is easy to observe that the MagFace magnitude is positively correlated with confidence of class center on the IJB-B-1845 benchmark.
This result reflects that the MagFace feature exhibits the expected within-class structure, where high quality samples distribute around class center while low quality ones are far away from the center.

% The x-axis indicates the confidences of being class centers following and y-axis indicates feature magnitudes.

\section{Conclusion}
\vspace{-1pt}

In this paper, we propose MagFace to learn unified features for face recognition and quality assessment.
By pushing ambiguous samples away from class centers, MagFace improves the within-class feature distribution from previous margin-based work for face recognition.
The adequate theoretical and experimental results convince that MagFace can simultaneously access quality for the input face image.
As a general framework, MagFace can be potentially extended to benefit other classification tasks such as fine-grained object recognition, person re-identification.
Moreover, the proposed principle of exploring feature magnitude paves the way to estimate quality for other objects, \eg, person body in reid or action snippet in activity classification.

{\small
  \bibliographystyle{ieee_fullname}
  \bibliography{egbib}
}

\clearpage
\begin{appendix}

\setlength{\tabcolsep}{5pt}

\begin{table*}
  \begin{center}
    % \resizebox{0.5\textwidth}{!}{%
    \footnotesize{
      \begin{tabular}{cccccc|ccc|ccccc}
        \hline
        Method & \multicolumn{5}{c|}{Hyperparameters} & \multicolumn{3}{c|}{Margin} & CFP-FP & \multicolumn{4}{c}{IJB-C (TAR@FAR)} \\
        \cline{11-14} & $l_m$ & $u_m$ & $\lambda_g$ & $l_a$ & $u_a$ & mean & max & min  &     & 1e-6 & 1e-5 & 1e-4 & 1e-3 \\
        \hline
        ArcFace &  - &- &- & - & - & 0.50 & - & - & 97.32 & 83.88 & 91.59  & 95.00  & 96.86 \\
        \hline
        MagFace & 0.45 & 0.65 & 35 & 10 & 110 & 0.50 & 0.49 & 0.52  & 97.23 & 81.12 & 91.44 & 94.95 &  96.96  \\
        \rowcolor{Gray1}   \cellcolor{white} & 0.40 & 0.80 & 35 & 10 & 110 & 0.50  & 0.46 & 0.53 &  \textbf{97.47} & \textbf{85.82} & \textbf{92.06} & \textbf{95.12} & 96.92 \\
               & 0.35 & 1.00 & 35 & 10 & 110 & 0.50 & 0.42 & 0.54 & 97.40 & 84.35 & 91.65 &  95.05 & \textbf{97.02}  \\
        \rowcolor{Gray1}   \cellcolor{white}  & 0.25 & 1.60 & 35 & 10 & 110 & 0.50 & 0.35 & 0.61 & 97.30 & 81.64 & 91.09 & 94.91 & 96.87  \\
        \hline
      \end{tabular}
    }
  \end{center}
  \caption{Verification accuracy (\%) on CFP-FP and IJB-C with different distributions of margins. Backbone network: ResNet50.}      \label{table:ablation}
\end{table*}
\setlength{\tabcolsep}{1.4pt}

\section{Proofs for MagFace}

Recall the MagFace loss for a sample $i$ is
\begin{equation}
  \resizebox{0.4\textwidth}{!}{$
    \begin{split}
      L_i & = -\log \frac{e^{s\cos{(\theta_{y_i}+ m(a_i))}}}{ e^{s\cos{(\theta_{y_i}+m(a_i))}} + \sum_{j=1, j\neq y_i}^{n}e^{s \cos{\theta_j}}} \\
      & \hspace{5em} + \lambda_g   g(a_i)
    \end{split}
    $}
  \label{eq:loss1}
\end{equation}
Let $A(a_i) = s \cos(\theta_{y_i}+ m(a_i))$ and $B = \sum_{j=1, j\neq y_i}^{n}e^{s \cos{\theta_j}}$ and rewrite the loss as
\begin{equation}
  \resizebox{0.25\textwidth}{!}{$
    L_i = - \log \frac{e^{A(a_i)}}{e^{A(a_i)} + B} + \lambda_g g(a_i)
    $}
\end{equation}
We first introduce and prove Lemma~\ref{prop:1}.

\begin{lemma}
  \label{prop:1}
  Assume that $f_i$ is top-k correctly classified and $m(a_i)\in [0, \pi/2]$. If the number of identities $n$ is much larger than $k$ (\textit{i.e.}, $n\gg k$), the probability of $\theta_{y_i} + m(a_i) \in [0, \pi/2]$ approaches  1.
\end{lemma}
\begin{proof}
  Denote the angle between feature $f_i$  and center class $W_j, j \in \{ 1, \cdots, n\}$  as $\theta_j$. Assuming the distribution of $\theta_j$ is uniform, it's easy to prove $P\left(\theta_j+m(a_i) \in [0, \pi/2]\right)=\frac{\pi/2-m(a_i)}{\pi}$. Let $p=\frac{\pi/2-m(a_i)}{\pi}$. If $f_i$ is top-k correctly classified,  the probability of $\theta_{y_i} + m(a_i) \in [0, \pi/2]$ is the same as the probability of there are at least k $\theta$ to satisfy $\theta + m(a_i) \in [0, \pi/2]$. Then the probability is
  \begin{equation}
    \resizebox{0.4\textwidth}{!}{$
      \begin{split}
        P\left(\theta_{y_i}+m(a_i)\in [0, \pi/2]\right)
        &= \sum_{i=k}^n \binom{n}{i} p^i(1-p)^{(n-i)} o\\
        & = 1 - \sum_{i=0}^{k-1} \binom{n}{i} p^i(1-p)^{(n-i)} \\
      \end{split}$}
  \end{equation}
  When $n$ is a large integer and $n\gg k$,  each $\binom{n}{i} p^i(1-p)^{(n-i)}, i=1, 2, \cdots k-1$ converges to 0. Therefore, probability of $\theta_{y_i} + m(a_i) \in [0, \pi/2]$ approaches 1.
\end{proof}
Lemma~\ref{prop:1} is fundamental  for the following proofs. The number of identities is large in real-world applications (\textit{e.g.}, 3.8M for MS1Mv2~\cite{guo2016ms,Deng2018}). Therefore, the probability of $\theta_{y_i}+ m(a_i)\in [0, \pi/2]$  approaches 1 in most cases.

\subsection{Requirements for MagFace}
In MagFace, $m(a_i), g(a_i), \lambda_g$ are required to have the following properties:
\begin{enumerate}
\item $m(a_i)$ is an increasing convex function in $[l_a, u_a]$ and $m'(a_i) \in (0, K]$, where $K$ is a upper bound;
\item $g(a_i)$ is a strictly convex function with $g'(u_a)=0$;
\item $\lambda_g \geq \frac{sK}{-g'(l_a)}$.
\end{enumerate}

\subsection{Proof for Property of Convergence}
We prove the property of convergence by showing the strict convexity of the function $L_i$ (Property~\ref{prop:convex}) and the existence of the optimum (Property~\ref{prop:uniq}). 
\begin{prop}
  \label{prop:convex}
  For $a_i \in [l_a, u_a]$, $L_i$ is a strictly convex function of $a_i$.
\end{prop}
\begin{proof}
  The first and second deriviates of $A(a_i)$ are
  \begin{equation}
    \resizebox{0.35\textwidth}{!}{$
      \begin{aligned}
        A'(a_i) & = -s\sin(\theta_{y_i}+ m(a_i)) m'(a_i) \\
        A''(a_i) & = -s\cos(\theta_{y_i}+ m(a_i)) (m'(a_i))^2  \\
        & \qquad - s\sin(\theta_{y_i}+ m(a_i)) m''(a_i)
      \end{aligned}
      $}
  \end{equation}
  According to Lemma~\ref{prop:1}, we have  $\cos(\theta_{y_i}+ m(a_i)) \geq 0$ and  $\sin(\theta_{y_i}+ m(a_i)) \geq 0$.  Because we define $m(a_i)$ to be convex and $g(a_i)$ to be strictly convex for $a_i \in [l_a, u_a]$,  $m''(a_i) \geq 0$  and $g''(a_i) > 0$ always hold. Therefore, $A''(a_i) \leq 0$.

  The first and second order derivatives of the loss $L_i$ are
  \begin{equation*}
    \resizebox{0.45\textwidth}{!}{$
      \begin{aligned}
        \frac{\partial L_i}{\partial a_i} & = -\frac{B}{e^{A(a_i)}+B} A'(a_i) + \lambda_g g'(a_i) \\
        \frac{\partial^2 L_i}{(\partial a_i)^2} &= -\frac{B}{(e^{A(a_i)}+B)^2} \left ( (e^{A(a_i)}+B) A''(a_i)  -  B  e^{A(a_i)} A'(a_i)^2 \right ) \\
        & \qquad + \lambda_g g''(a_i) \\
        &= -\frac{B}{e^{A(a_i)}+B} A''(a_i)  + \frac{B^2}{(e^{A(a_i)}+B)^2}  e^{A(a_i)} A'(a_i)^2 \\
        & \qquad + \lambda_g g''(a_i)
      \end{aligned}
      $}
  \end{equation*}

   As $B>0, e^{A(a_i)}+B>0$, it's easy to prove that first two parts of $\frac{\partial^2 L_i}{(\partial a_i)^2} $ are non-negative while the third part is always positive. Therefore, $\frac{\partial^2 L_i}{(\partial a_i)^2}>0$ and $L_i$ is a strictly convex function with respect to $a_i$.
\end{proof}

\begin{prop}
  \label{prop:uniq}
  A unique optimal solution $a_i^*$ exists in $[l_a, u_a]$.
\end{prop}

\begin{proof}
  Because the loss function $L_i$ is a strictly convex function, we have $\frac{\partial L_i}{\partial a^1_i} > \frac{\partial L_i}{\partial a^2_i}$ if $u_a \geq a^1_i > a^2_i \geq l_a$. Next we prove that there exist a optimal solution  $a^*_i\in [l_a, u_a]$. If it exists, then it is unique because of the strict convexity.

  As $ \frac{\partial L_i}{\partial a_i}(a_i) = \frac{Bs}{e^{A(a_i)}+B}\sin(\theta_{y_i}+ m(a_i)) m'(a_i)  + \lambda_g g'(a_i)$ and considering the constraints $m'(a_i) \in (0, K]$, $g'(u_a) = 0$, $\lambda_g \geq \frac{sK}{-g'(l_a)}$,  the values of derivatives of $l_a, u_a$ are
  \begin{equation}
    \resizebox{0.4\textwidth}{!}{$
      \begin{split}
        \frac{\partial L_i}{\partial a_i}(u_a) & = \frac{Bs}{e^{A(a_i)}+B}\sin(\theta_{y_i}+ m(a_i))m'(u_a) > 0 \\
        \frac{\partial L_i}{\partial a_i}(l_a) & = \frac{Bs}{e^{A(a_i)}+B}\sin(\theta_{y_i}+ m(a_i))m'(l_a) + \lambda_g g'(l_a) \\
        & < sK + \lambda_g g'(l_a) \leq 0
      \end{split}
      $}
  \end{equation}
  As $\frac{\partial L_i}{\partial a_i}$ is monotonically and strictly increasing, there must exist a unique value in $[l_a, u_a]$ which have a 0 derivative. Therefore, an optimal solution exists and is unique.
\end{proof}

\subsection{Proof for Property  of Monotonicity}
To prove the property of monotonicity, we first show that optimal $a_i^*$ increases with a smaller cosine-distance to its class center (Property~\ref{prop:mono1}).
As $B$ can reveal the overall cos-distances to other class centers, we further prove that decreasing $B$ (distances to other class centers increases) can increase  optimal feature magnitude (Property~\ref{prop:mono2}).
In the end, we can conclude that $a_i^*$ is monotonically increasing as the cosine-distance to its class center decreases and the cosine-distances to other classes increase.
\begin{prop}
  \label{prop:mono1}
  With fixed $f_i$ and $W_j, j\in\{1,\cdots, n\}, j\neq y_i$, the optimal feature magnitude $a^*_i$ is monotonically increasing if the cosine-distance to its class center $W_{y_i}$ decreases.
\end{prop}
\begin{proof}
  Assuming there are two class center $W^1_{y_i}, W^2_{y_i}$ and their cosine distances to feature $f_i$ are $\theta_{y_i}^1,\theta_{y_i}^2$. Assuming $\theta_{y_i}^1 < \theta_{y_i}^2$ (\textit{i.e.}, class center $W^1_{y_i}$ has a smaller distance with feature $f_i$) and the corresponding optimal feature magnitudes are $a_{i,1}^*, a_{i,2}^*$.

  % Because of its strictly convexity, $\frac{\partial L_i}{\partial a_i}=0$ has a unique solution.
  The first derivate of $L_i$ is
  \begin{equation}
    \resizebox{0.4\textwidth}{!}{$
      \begin{aligned}
        \frac{\partial L_i}{\partial a_i} & = -\frac{B}{e^{A(a_i)}+B} A'(a_i) + \lambda_g g'(a_i) \\
        &= \frac{Bs m'(a_i)}{e^{s \cos(\theta_{y_i}+ m(a_i))}+B}\sin(\theta_{y_i}+ m(a_i)) + \lambda_g g'(a_i)
      \end{aligned}
      $}
  \end{equation}
  For $\theta_{y_i}+ m(a_i)\in (0, \pi/2]$, we have  $\cos(\theta_{y_i}^1+ m(a_i)) > \cos(\theta_{y_i}^2+ m(a_i))$ and $\sin(\theta_{y_i}^1+ m(a_i)) < \sin(\theta_{y_i}^2+ m(a_i))$. With $m'(a_i) > 0$, it's obvious that
  \begin{equation*}
    \resizebox{0.5\textwidth}{!}{$
      \frac{Bs m'(a_i)}{e^{s \cos(\theta^1_{y_i}+ m(a_i))}+B}\sin(\theta^1_{y_i}+ m(a_i))  < \frac{Bs m'(a_i)}{e^{s \cos(\theta^2_{y_i}+ m(a_i))}+B}\sin(\theta^2_{y_i}+ m(a_i)).
      $}
  \end{equation*}
  Therefore, we have $\frac{\partial L_i(\theta^1_{y_i})}{\partial a_i} < \frac{\partial L_i(\theta^2_{y_i})}{\partial a_i}$. Based on the property of optimal solution for strictly convex function, we have $0 = \frac{\partial L_i(\theta^1_{y_i})}{\partial a_{i,1}^*} = \frac{\partial L_i(\theta^2_{y_i})}{\partial a_{i,2}^*} > \frac{\partial L_i(\theta^1_{y_i})}{\partial a_{i,2}^*}$, which leads to $a_{i,1}^* > a_{i,2}^*$.
\end{proof}

\begin{prop}
  \label{prop:mono2}
  With other things fixed, the optimal feature magnitude $a^*_i$ is monotonically increasing with a decreasing $B$ (\ie, increasing inter-class distance).
\end{prop}
\begin{proof}
  Assume $0< B_1 < B_2$ with optimum $a_{i,1}^*, a_{i,2}^*$. Similar to the proof before, it's easy to show
  \begin{equation*}
    \resizebox{0.5\textwidth}{!}{$
      \frac{B_1 s m'(a_i)}{e^{s \cos(\theta_{y_i}+ m(a_i))}+B_1 }\sin(\theta_{y_i}+ m(a_i))  < \frac{B_2 s m'(a_i)}{e^{s \cos(\theta_{y_i}+ m(a_i))}+B_2 }\sin(\theta_{y_i}+ m(a_i)).
      $}
  \end{equation*}
  Therefore, we have $\frac{\partial L_i(B_1)}{\partial a_i} < \frac{\partial L_i(B_2)}{\partial a_i}$. Based on the property of optimal solution for strictly convex function, we have $0 = \frac{\partial L_i(B_1)}{\partial a_{i,1}^*} = \frac{\partial L_i(B_2)}{\partial a_{i,2}^*} > \frac{\partial L_i(B_1)}{\partial a_{i,2}^*}$, which leads to $a_{i,1}^* > a_{i,2}^*$.
\end{proof}

\section{Experimental Settings}

\subsection{Training settings for Figure~3}
We adopt ResNet50 as the backbone network. Models are trained on MS1Mv2 \cite{guo2016ms,Deng2018} for  20 epochs with batch size 512 and initial learning rate 0.1, dropped by 0.1 every 5 epochs. 512 samples of the last iteration are used for visualization.

\begin{figure*}
  \centering
  \includegraphics[width=0.95\textwidth]{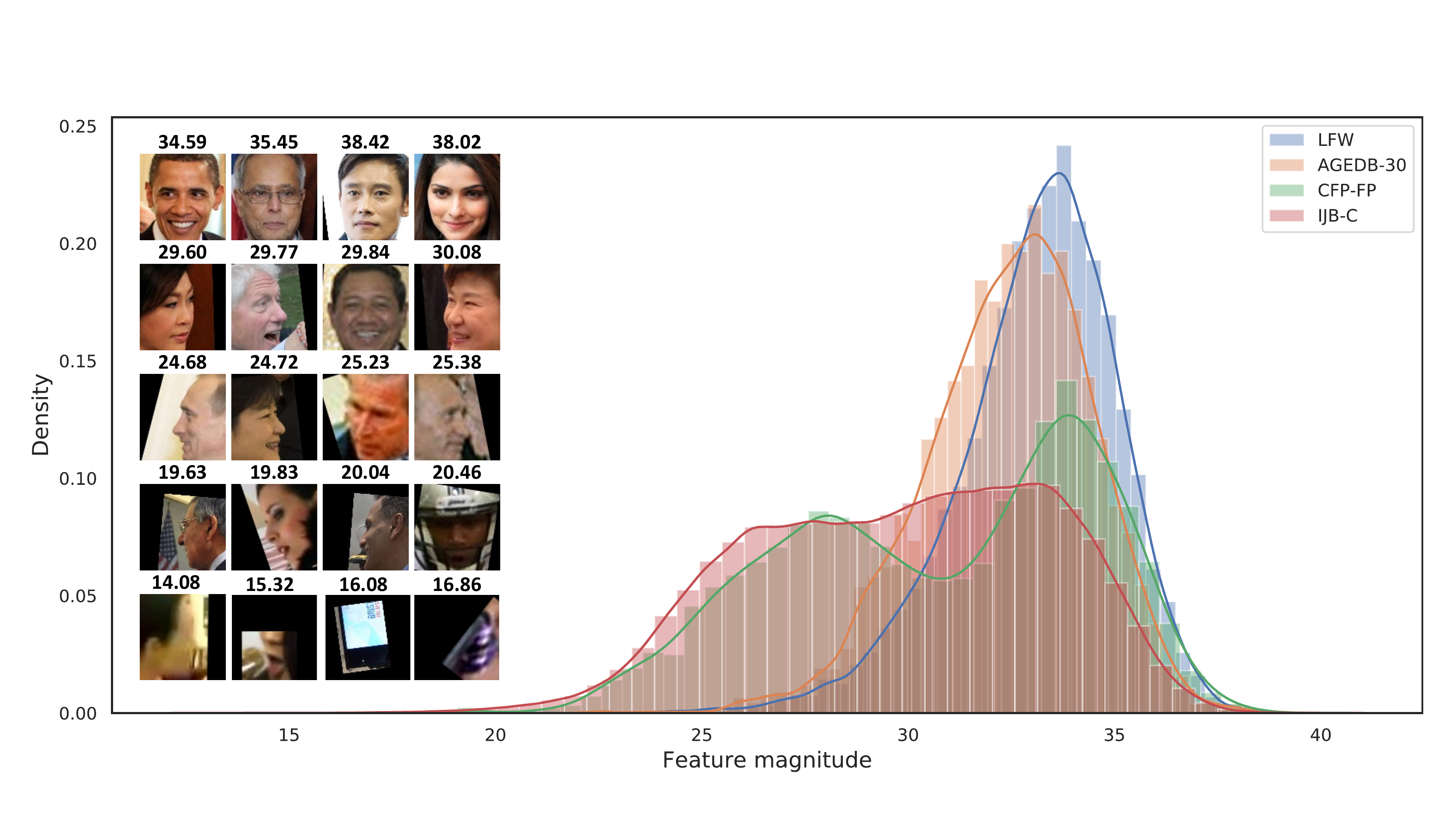}
  \caption{Extended Visualization of Figure~6. }
  \label{fig:dist_qlt_appendix}
\end{figure*}
\setlength{\tabcolsep}{4pt}

\subsection{Settings of $m(a_i)$, $g(a_i)$ and $\lambda_g$}
In our experiments, we define function $m(a_i)$ as a linear function defined on $[l_a, u_a]$ with $m(l_a) = l_m, m(u_a) = u_m$ and  $g(a_i) = \frac{1}{a_i} + \frac{1}{u_a^2}a_i$. Therefore, we have
\begin{equation}
  \resizebox{0.25\textwidth}{!}{$
    \begin{aligned}
     & m(a_i) = \frac{u_m-l_m}{u_a-l_a}(a_i-l_a) + l_m \\
      &\lambda_g \geq \frac{sK}{-g'(l_a)} =  \frac{su_a^2l_a^2}{(u_a^2-l_a^2)} \frac{u_m-l_m}{u_a-l_a}
    \end{aligned}
    $}
  \label{eq:mg}
\end{equation}

\section{Ablation Study on Margin Distributions} \label{ap:ablation}
In this section, effects of the feature distributions during training are studied. With $(\lambda_g, l_a, u_a)$ fixed to $(35, 10, 110)$,  we carefully select various combinations of $l_m, u_m$ to align the mean margin on the training dataset to ArcFace (0.5) in our implementation. Features are distributed more separated if with a larger maximum margin and a smaller minimum margin.

Table~\ref{table:ablation} shows the recognition results with various hyper-parameters. With $(l_m, u_m)=(0.45, 0.65)$, the penalty of magnitude loss degrades the performance of the recognition. With $(l_m, u_m)=(0.25, 1.60)$, the performance is also worse than then baseline as hard samples are assigned to small margins (\textit{a.k.a.}, hard/noisy samples are down-weighted).  Parameter $(0.40, 0.80)$ balances the feature distribution and margins for hard/noisy samples,  and therefore achieves a significant improvement on benchmarks.

\section{Extended Visualization of Figure~6}
We present a extended visualization of figure~6 in figure~\ref{fig:dist_qlt_appendix} which has more examples of faces with feature magnitudes.
All the faces are sample from the IJB-C benchmark.
It can be seen that faces with magnitudes around 28 are mostly profile faces while around 35 are high-quality and frontal faces.
That is consistent with the profile/frontal peaks in the CFP-FP benchmark and indicates that faces with similar magnitudes show similar quality patterns across benchmarks.
In real applications, we can set a proper threshold for the magnitude and should be able to filter similar low-quality faces, even under various scenarios.

Besides directly served as qualities, our feature magnitudes can also be used as quality labels for faces, which avoids human labelling costs.
These labels are more suitable for recognition, and therefore can be used to boost other quality models.

% \section{Lemma for Softmax}
% \begin{lemma}
%   \label{lemma:1}        
%   By fixing weight vectors and directions of the feature vectors, softmax loss is a monotonically decreasing function with the increasing of the features’ L2-norm as long as the features are correctly classified.

% \end{lemma}
% \begin{proof}
%   Proofs in NormFace~\cite{Wang2017}. 
% \end{proof}

\section{Mag-CosFace}

In the main text, MagFace is modified from the ArcFace loss.
In this section, we show that MagFace based on CosFace loss (denote as Mag-CosFace) can theoretically achieve the same effects.
Mag-CosFace loss for a sample $i$ is
\begin{equation}
  \resizebox{0.4\textwidth}{!}{$
    \begin{split}
      L_i & = -\log \frac{e^{s(\cos{\theta_{y_i}- m(a_i)})}}{ e^{s(\cos{\theta_{y_i}- m(a_i)})} + \sum_{j=1, j\neq y_i}^{n}e^{s \cos{\theta_j}}} \\
      & \hspace{5em} + \lambda_g   g(a_i)
    \end{split}
    $}
  \label{eq:loss2}
\end{equation}
Let $A(a_i) = s (\cos\theta_{y_i}- m(a_i))$ and $B = \sum_{j=1, j\neq y_i}^{n}e^{s \cos{\theta_j}}$ and rewrite the loss as
\begin{equation}
  \resizebox{0.25\textwidth}{!}{$
    L_i = - \log \frac{e^{A(a_i)}}{e^{A(a_i)} + B} + \lambda_g g(a_i)
    $}
\end{equation}

\subsection{Property of Convergence for Mag-CosFace}
\begin{prop}
  \label{prop:convex}
  For $a_i \in [l_a, u_a]$, $L_i$ is a strictly convex function of $a_i$.
\end{prop}
\begin{proof}
  The first and second deriviates of $A(a_i)$ are
  \begin{equation}
    % \resizebox{0.35\textwidth}{!}{$
    \small
      \begin{aligned}
        A'(a_i) & = -sm'(a_i) \\
        A''(a_i) & = -sm''(a_i)
      \end{aligned}
      % $}
    \end{equation}
    As $A''(a_i)\leq 0$, the property can be proved following that presented before.
 % Because we define $m(a_i)$ to be convex and $g(a_i)$ to be strictly convex for $a_i \in [l_a, u_a]$,  $m''(a_i) \geq 0$  and $g''(a_i) > 0$ always hold. Therefore, $A''(a_i) \leq 0$.
 %  The first and second order derivatives of the loss $L_i$ are
 %  \begin{equation*}
 %    \resizebox{0.45\textwidth}{!}{$
 %      \begin{aligned}
 %        \frac{\partial L_i}{\partial a_i} & = -\frac{B}{e^{A(a_i)}+B} A'(a_i) + \lambda_g g'(a_i) \\
 %        \frac{\partial^2 L_i}{(\partial a_i)^2} 
 %        &= -\frac{B}{e^{A(a_i)}+B} A''(a_i)  + \frac{B^2}{(e^{A(a_i)}+B)^2}  e^{A(a_i)} A'(a_i)^2 \\
 %        & \qquad + \lambda_g g''(a_i)
 %      \end{aligned}
 %      $}
 %  \end{equation*}

 %   As $B>0, e^{A(a_i)}+B>0$, it's easy to prove that first two parts of $\frac{\partial^2 L_i}{(\partial a_i)^2} $ are non-negative while the third part is always positive. Therefore, $\frac{\partial^2 L_i}{(\partial a_i)^2}>0$ and $L_i$ is a strictly convex function with respect to $a_i$.
\end{proof}

\begin{prop}
  \label{prop:uniq}
  A unique optimal solution $a_i^*$ exists in $[l_a, u_a]$.
\end{prop}

\begin{proof}
  We only need to prove 
  \begin{equation}
    \resizebox{0.3\textwidth}{!}{$
    \begin{split}
        \frac{\partial L_i}{\partial a_i}(u_a) & = \frac{Bs}{e^{A(a_i)}+B}m'(u_a) > 0 \\
        \frac{\partial L_i}{\partial a_i}(l_a) & = \frac{Bs}{e^{A(a_i)}+B}m'(l_a) + \lambda_g g'(l_a) \\
        & < sK + \lambda_g g'(l_a) \leq 0
      \end{split}
      $}
  \end{equation}
  Then it's easy to have there is a unique optimu.
\end{proof}

\subsection{Property  of Monotonicity for Mag-CosFace}
\begin{prop}
  With fixed $f_i$ and $W_j, j\in\{1,\cdots, n\}, j\neq y_i$, the optimal feature magnitude $a^*_i$ is monotonically increasing if the cosine-distance to its class center $W_{y_i}$ decreases.
\end{prop}
\begin{proof}

  The first derivate of $L_i$ is
  \begin{equation}
    \resizebox{0.3\textwidth}{!}{$
      \begin{aligned}
        \frac{\partial L_i}{\partial a_i} & = -\frac{B}{e^{A(a_i)}+B} A'(a_i) + \lambda_g g'(a_i) \\
        &= \frac{Bs m'(a_i)}{e^{s (\cos\theta_{y_i}- m(a_i))}+B} + \lambda_g g'(a_i)
      \end{aligned}
      $}
  \end{equation}
  For $\theta_{y_i}^1 < \theta_{y_i}^2$, the core here is to  prove $\frac{\partial L_i(\theta^1_{y_i})}{\partial a_i} < \frac{\partial L_i(\theta^2_{y_i})}{\partial a_i}$, which is obvious as $\cos\theta_{y_i}^1 > \cos\theta_{y_i}^2$.
  The rest of the proofs is the same as those for the original MagFace.  
% we have $\frac{\partial L_i(\theta^1_{y_i})}{\partial a_i} < \frac{\partial L_i(\theta^2_{y_i})}{\partial a_i}$  
%   \begin{equation*}
%     \resizebox{0.4\textwidth}{!}{$
%       \frac{Bs m'(a_i)}{e^{s \cos(\theta^1_{y_i}+ m(a_i))}+B}  < \frac{Bs m'(a_i)}{e^{s \cos(\theta^2_{y_i}+ m(a_i))}+B}.
%       $}
%   \end{equation*}
\end{proof}

\begin{prop}
  With other things fixed, the optimal feature magnitude $a^*_i$ is monotonically increasing with a decreasing $B$ (\ie, increasing inter-class distance).
\end{prop}
\begin{proof}
  It's easy to have $\frac{\partial L_i(B_1)}{\partial a_i} < \frac{\partial L_i(B_2)}{\partial a_i}$ if $B_1 < B_2$.
  The rest of the proofs is the same as those for the original MagFace.
\end{proof}

\section{Authors' Contributions}
Shichao Zhao and Zhida Huang contribute similarly to this work.
Besides involved in discussions, Shichao Zhao mainly conducted experiments on face clustering and Zhida Huang implemented baselines as well as evaluation metrics for quality experiments.
\end{appendix}
\end{document}